\documentclass[sigconf, nonacm]{acmart}

\usepackage[title]{appendix}
\usepackage{mathtools}
\usepackage{graphicx}
\usepackage{subcaption}
\usepackage{paralist} 

\newcommand{\reals}{\ensuremath{\mathbb{R}}}
\newcommand{\argmax}{\mathop{\mathrm{arg\,max}}}

\newcommand{\myparagraph}[1]{\smallskip\noindent{\textbf{#1}}}

\begin{document}
\title{Hyperdimensional computing for structured querying on tabular data embeddings}

\author{Sebastián Bugedo}
\affiliation{\institution{UHasselt, DSI}
\city{Diepenbeek}
  \state{Belgium}
}
\email{sebastian.bugedo@uhasselt.be}

\author{Stijn Vansummeren}
\affiliation{\institution{UHasselt, DSI}
\city{Diepenbeek}
  \state{Belgium}
}
\email{stijn.vansummeren@uhasselt.be}

\begin{abstract}

  Tabular data embeddings have become a cornerstone of data profiling and data
  integration pipelines, enabling tasks such as entity annotation and
  resolution; schema matching; column type detection; and table search, among
  others. Existing approaches embed rows, columns, or entire tables into a
  vector space and rely on nearest-neighbor search to retrieve candidate
  matches. A fundamental limitation of current embedding methods is the lack of
  interpretable similarity scores: the concrete similarity value between a query
  and its nearest neighbour carries no intrinsic meaning, making it impossible
  to determine whether that neighbour is a true match or simply the
  least-dissimilar item in a corpus that contains no valid answer. This
  inability to set principled thresholds for retrieval undermines practical
  deployment, particularly for zero-match detection.

We investigate the use of HyperDimensional Computing (HDC), specifically the
Holographic Reduced Representations (HRR) model, as a framework for tabular row
embeddings when the retrieval task corresponds to answering structured
select-project queries in vector space. Exploiting the algebraic properties of
HDC operations, we derive closed-form expected similarity values for both
equality and non-equality retrieval predicates, which converge to interpretable
values as dimensionality increases, and use these to identify suitable retrieval
thresholds. We evaluate HDC against EmbDI, a graph-based baseline, on two
real-world datasets across varying table sizes and predicate lengths. Our
results show that HDC matches or outperforms EmbDI for row retrieval across all
configurations, handles non-equality predicates more robustly, and achieves
perfect attribute projection accuracy at sufficient dimensionality---while
uniquely enabling reliable identification of zero-match predicates through its
principled thresholds.
\end{abstract}

\maketitle

\pagestyle{plain}

\begingroup\small\noindent\raggedright\textbf{Reproducibility}\\
The source code, data, and/or other artifacts have been made available at \url{https://github.com/UHasselt-DSI-Data-Systems-Lab/code-hdc-for-tabular-data}.
\endgroup

\section{Introduction}

Tabular data appears ubiquitously in both private enterprise data repositories
and on the public Web. In both of these settings, however, tabular data is often
found to be very heterogeneous, noisy or incomplete, as well as lacking
documentation to guide its correct interpretation. As a consequence, extensive
preprocessing that profiles the data, cleans it, or integrates it with other
sources has become essential to extract valuable insight from tables ``in the
wild''. Prime examples of such preprocessing tasks are Schema Mapping~\cite{parciak2025schema}, the
related task of Semantic Type Detection (also called Column Type Annotation)~\cite{hulsebos2019sherlock},
Entity Annotation, Entity Resolution~\cite{christophides2020entity}, and Table Search~\cite{nargesian2018table}, among
others.

In the past years, the increasing capabilities and popularity of deep learning
have motivated the development of machine-learning-based approaches for these
tasks~\cite{capuzzo2020embdi,hulsebos2019sherlock,zhang2020sato,li2020deep}.
At their core lies the \emph{embedding} of
tabular data into vector space, either implicitly or explicitly, with the aim to
augment the explicitly stated data with a form of (task-dependent) semantic
understanding~\cite{capuzzo2020embdi, khatiwada2023santos, dong2023deepjoin}.  In general, a \emph{vector embedding} for a class
$\mathcal{O}$ of objects is a mapping $f$ from $\mathcal{O}$ into some vector
space, typically $\reals^d$ for some finite dimension $d$. The idea is to
construct a vector embedding in such a way that geometric relationships in the
vector space reflect semantic relationships between the objects in
$\mathcal{O}$. Importantly, we want similar objects in $\mathcal{O}$ to be
mapped to vectors that are close to one another with respect to some standard
metric on the vector space. This allows to find, given a query object,
semantically similar objects in a corpus of known objects using (approximate)
nearest neighbor search. In entity annotation, for example, one may hence
compute a vector embedding for a single table cell (possibly with information
from neighboring cells in the table) and find the nearest neighbors in a known
corpus to derive candidate entities to link to, after which these may be passed
to other machine learning models or logic-based systems for completing the task
in a RAG-like fashion. Similarly, in column type detection, one may embed entire
table columns, and in entity resolution one may embed table rows.

Different tabular data embeddings methods hence attempt to capture the
information and relationships in tables at different levels of granularity,
producing embeddings for cells, rows, columns or sets thereof---possibly entire
tables. An important desideratum of tabular embeddings for many of the
above-cited tasks is \emph{interpretability of similarity scores}.  Indeed, while
nearest neighbor search allows to retrieve candidate matches ranked by distance,
by itself this does not specify when none of the retrieve candidates are ``good
enough''. For example, a cosine similarity of 0.72 between a query object and
its nearest neighbour could mean that the match is excellent (if the embedding
space is well-calibrated and 0.72 is high in context), or that the match is poor
(if the true answer simply doesn't exist in the corpus, and 0.72 is the score of
the closest wrong answer). The distinction has practical relevance. In the
SemTab Entity annotation challenge~\cite{semtab2025}, for example, one
needs to correctly identify cells for which no corresponding entity in a
knowledge graph is known. Likewise, in column type detection, one needs to
recognize when no suitable pre-defined type is known, and in Table Union
search, one needs to recognize when no suitable unionable table is known.

For task like these, we are hence in search of \emph{interpretable} embedding
methods that specify similarity thresholds below which candidate matches should
be ignored. By contrast,  most contemporary embedding methods are opaque in this
sense. Indeed, broadly speaking we can identify two classes of embedding
methods: (i) those that represent cells, rows, and/or columns as nodes in a
graph and use graph representation learning to derive the
embeddings~\cite{capuzzo2020embdi,cucumides2025features,chen2023hytrel,tchuitcheu2024table}; and (ii)
those that serialize the table as a text string, and feed this into an LLM for
attention-based embedding~\cite{deng2022turl,yin2020tabert,herzig2020tapas, iida2021tabbie}. Neither class provides generally interpretable distance scores.

We postulate that \emph{hyperdimensional computing} (HDC)~\cite{kanerva2009hyperdimensional,Thomas2021perspective} may offer a general framework for designing interpretable vector embeddings of tabular data. 
HDC is a symbolic approach to representation learning that encodes structured information in a high-dimensional space through simple vector operations: \textit{binding}, \textit{bundling}, and \textit{unbinding}. It has been successfully applied in various domains, including natural language processing, computer vision, and robotics, and thanks to its efficient computational properties, it offers a compelling alternative to traditional neural approaches~\cite{kleyko2022surveyi}. Moreover, it provides a principled way to define similarity thresholds based on the algebraic properties of the operations, as well as an interpretable framework to extract information from encoded entities with probabilistic guarantees~\cite{Thomas2021perspective}. 

In this paper, we make an initial study of the the suitability of HDC for
defining tabular embeddings, focusing in particular on row embeddings. We draw
inspiration from the work of Mellouli and Papotti~\cite{mellouli2025selection},
who argue that, for row embeddings, a reasonable notion of ``semantic
similarity'' corresponds to being able to answer simple select-project queries
directly on the embedded vector space instead of the original tabular data. To
illustrate, consider a table with schema $\{A,B,C\}$. Then finding the row most
similar to (the embedding of) the partial row $(A\colon a, B\colon b)$
corresponds to executing the relational algebra selection query
$\sigma_{A = a, B = b}(T)$ on the input table $T$. Moreover, given the embedding
of a row in $T$ and an attribute $A$, finding the embedding of the value in the
row at attribute $A$ corresponds to being able to execute projection queries on
embeddings.  In particular, Mellouli and Papotti showcase that embeddings
generated by the EmbDI embedding framework have this property of being able to
execute structural queries on them, while it is not clear how to do this directly 
over LLM-based approaches. 
EmbDI is a graph-based embedding method designed for general data
integration tasks that does not require extensive
training~\cite{capuzzo2020embdi}. Therefore, we consider it as a suitable baseline to compare HDC to, as it offers a high-level definition for row embeddings, and it is not LLM-based.

Our contributions are threefold:
\begin{compactenum}[(1)]
\item We discuss how HDC can be used for row-embeddings that allow structured select-project queries.
\item We theoretically derive thresholds for such queries for HDC,  allowing to distinguish matches from non-matches.
\item We empirically compare the performance of HDC to EmbDI for structured
  querying. Our results illustrate that HDC matches or outperforms EmbDI for row
  retrieval across all configurations; handles non-equality predicates more
  robustly; and improves attribute projection accuracy, especially when the
  projected column has low cardinality. Furthermore, HDC enables
  reliable identification of zero-match predicates through its principled
  thresholds, something that is not available for EmbDI.
\end{compactenum}

Overall, this illustrates the suitability of HDC for row-based embeddings. Its appropriateness for other kinds of embeddings remains to be studied; we discuss necessary future work in this respect in Section~\ref{sec:discussion}.

This paper is further organized as follows. Section~\ref{sec:related-work} discusses related work while Section~\ref{sec:preliminaries}  introduces the problem statement and necessary background. Section~\ref{sec:hdc} introduces HDC; discusses how to use this for structured querying; and derives the theoretical threshold values. Experimental setup is discussed in Section \ref{sec:experimental-setting} and experimental results in Section~\ref{sec:results}. We conclude and discuss future work in Section~\ref{sec:discussion}.
The interested reviewer may find the proofs of formal statements, as well as additional experimental results in the Appendix. The Appendix is not to be considered part of the submission, and can be read at the reviewers' discretion.

\section{Related work}
\label{sec:related-work}

\textbf{Tabular data embeddings.} Methods for embedding tabular data fall for the most part into two families. Graph-based approaches construct a graph from table entities and apply representation learning: EmbDI~\cite{capuzzo2020embdi} builds a tripartite graph over rows, columns, and values and trains word embeddings on random walks; HyTrel~\cite{chen2023hytrel} instead uses a hypergraph to represent the table entities and relationships, to then embed using transformer layers; and Tchuitcheu et al.~\cite{tchuitcheu2024table} propose heterogeneous graph embeddings that encode cell positions. LLM-based approaches serialize tables as text and exploit attention mechanisms: TaBERT~\cite{yin2020tabert}, TaPas~\cite{herzig2020tapas}, TURL~\cite{deng2022turl}, and TABBIE~\cite{iida2021tabbie} all fall in this category. Is it not clear how specific similarity values between vectors generated by these methods can be interpreted, and thus, how to set thresholds for retrieval tasks. Our work applies HDC as a third, symbolic alternative that admits principled threshold derivation.

\myparagraph{Universal embeddings.} Several works aim at embeddings that generalize across datasets for downstream integration tasks. Joinable table discovery~\cite{dong2023deepjoin} rely on column-level similarity; schema matching~\cite{parciak2025schema} operates at the schema level; entity resolution~\cite{christophides2020entity} and column type annotation~\cite{zhang2020sato} require cell- and column-level embeddings respectively. However, these approaches are task-specific, and similarity values cannot be generalized when used in other settings. EmbDI~\cite{capuzzo2020embdi} was designed to span multiple tasks with a single embedding, while the authors in~\cite{franz2025universal} propose a \textit{universal embedding for tabular data}, using a graph auto-encoder approach. HDC is a natural candidate for such settings given its lightweight, compositional framework to generate embeddings for different types of structures. In particular, we focus on row-based embeddings.

\myparagraph{Thresholding in HDC.} For the specific HDC model that we use and describe in Section~\ref{sec:hrr-model}, the HRR model, theoretical analyses have been conducted on similarity values, although for decoding tasks and vector distributions different from the ones considered in this work. We discuss the differences in detail in Section~\ref{sec:thresholds}.

\section{Preliminaries}
\label{sec:preliminaries}

In this section we introduce the notation and the problem statement, and then we present EmbDI, the embedding method that we use as baseline in our experiments.

\myparagraph{Notation.} We assume given two disjoint sets $\mathcal{A}$ and $\mathcal{B}$ of attribute names and values, respectively, and range over the elements of $\mathcal{A} \cup \mathcal{B}$ by lowercase letters (e.g. $a$, $b$, $c$). Vectors are denoted by bold lowercase letters (e.g. $\mathbf{a}$, $\mathbf{b}$, $\mathbf{c}$) with dimension $d$. Indexing is done with square brackets (e.g. $\mathbf{a}[i]$ for $0\leq i\leq d-1$) and whenever clear, $\mathbf{x}$ denotes the vector representation of $x$. We treat table rows as partial functions $r: \mathcal{A} \rightharpoonup \mathcal{B}$, and we denote the set of all rows by $\mathcal{R}$. For this paper, once a table $T$ with $m$ columns is specified, all rows come from that table.

\subsection{Structured querying in vector spaces}

\myparagraph{Problem statement.} We consider the problem of performing structured queries on a set of rows that have been embedded in a vector space. More concretely, for a table with $m\geq 1$ columns, we focus on two notions of querying:
\begin{compactitem}[--]
    \item \textit{Row retrieval}.  We define a \textit{selection predicate} of length $n$ as a partial function $q: \mathcal{A} \rightharpoonup \mathcal{B}$ such that $|\text{dom}(q)|=n$, also writing such predicates as
    $$q=(a_1:b_1, a_2:b_2, \ldots, a_n:b_n).$$
    Given a predicate, the goal is to retrieve the matching rows.
    According to the interpretation we give to $q$, we distinguish between
    \textit{equality predicates}, where we look for rows $r$ that satisfy
    $r(a_i)=b_i$ for every $1\leq i\leq n$, and \textit{non-equality
      predicates}, where we look for rows that satisfy $r(a_i)\neq b_i$ for
    every $1\leq i\leq n$. We denote by $\mathcal{Q}$ the set of all selection predicates.
    
    \item \textit{Attribute projection}. Given a row and an attribute name, the goal is to retrieve the value of the attribute in the row, i.e., for $r$ and $a$, we want to obtain $r(a)$.
\end{compactitem}
Both tasks are specified in terms of non-embedded entities, and our goal is to
perform them in vector space, after having performed row embedding and without
making use of the original table. For this, we assume that we have an embedding
function
$\varphi:\mathcal{A}\cup\mathcal{B}\cup\mathcal{R}\cup\mathcal{Q}\rightarrow\mathcal{H}$
that maps entities to some vector space $\mathcal{H}$. We further assume that we
can measure similarity between vectors in $\mathcal{H}$ using a similarity
function $\text{sim}:\mathcal{H}\times\mathcal{H}\rightarrow \mathbb{R}$. Both the embedding function and similarity measure are discussed in Sections~\ref{sec:embdi} and~\ref{sec:hrr-model}.

\myparagraph{Performance evaluation.} To assess how well can we recover original values from encoded structures, we evaluate the performance of the methods for both tasks using the following standard metrics.

For row retrieval, we measure precision, recall and F1-score as a function of a similarity threshold $\tau$, which is used to determine whether a row is selected or not. Specifically, for a table $T$ and predicate $q$ we contrast the set 
$$S(T, q, \tau)=\{r \in T \mid \text{sim}(\varphi(r), \varphi(q)) > \tau\}$$
against the set of rows that satisfy the selection predicate $q$ in the original table. Alternatively, we measure the same metrics for top-$k$ retrieval, where instead of using a threshold, we get the $k$ most similar rows to $\varphi(q)$, for some fixed value $k$.

For attribute projection, we measure accuracy of exact matches: for a row $r$ and attribute $a$, we obtain a candidate vector $\hat{\mathbf{v}}$ for the attribute value from $\varphi(r)$ and compare 
\begin{equation}\label{eq:projection}
    \hat{b}=\argmax_{b\in\mathcal{B}} \{\text{sim}(\hat{\mathbf{v}}, \varphi(b))\}
\end{equation}
to the ground truth $r(a)$, having an exact match if $\hat{b}=r(a)$. The exact mechanism to obtain $\hat{\mathbf{v}}$ from $\varphi(r)$ depends on the embedding framework, as described in Sections~\ref{sec:embdi} and~\ref{sec:encod-decod-hrr}.

\subsection{EmbDI}
\label{sec:embdi}

EmbDI~\cite{capuzzo2020embdi}, which is short for ``embedding for data
integration'', is a framework proposed to generate relational, local embeddings
for different entities in a tabular dataset: rows, column names and values. It
is motivated by the need of a general embedding method that spans across
different datasets, as well as adapting to downstream tasks such as schema
matching, entity resolution and token matching. In~\cite{mellouli2025selection},
Mellouli and Papotti study selection/projection queries on the tabular
embeddings generated by EmbDI, showing that the embeddings can be used to answer
such queries.values and the type of predicates.

EmbDI creates tabular embeddings in a three-step process. First, each row is identified with a unique ID. Then, table values and their relations are represented in a tripartite graph where nodes correspond to cell values, column names and row IDs. Cell value $b$ is linked to column name $a$ if it appears in a row $r$ with $r(a)=b$. Similarly, $b$ is linked to row ID $j$ if it appears in row $r$, and $j$ is the assigned ID for that row.  
Second, random walks are performed on this graph to generate sequences of nodes, with the goal of capturing the proximity between values. 
Lastly, these sequences are seen as sentences and are pooled to train a standard word embedding model. The default model, and the one employed in~\cite{mellouli2025selection} is word2vec with the Skip-Gram learning method.
The end result is an embedding function $\varphi$ that assigns an embedding vector to (1) each column name, (2) each pair $(a:b)$ of column name $a$ and value $b$, and (3) rows.

Using these base embeddings, the authors of~\cite{mellouli2025selection} perform row and attribute projection as follows.
\begin{compactitem}[--]
    \item For row retrieval, predicate  $q=(a_1:b_1, \ldots, a_n:b_n)$ is embedded as 
    $$\varphi(q)=\alpha\cdot\sum_{i=1}^n w_i\cdot\varphi(a_i,b_i)$$
    where $\varphi(a_i,b_i)$ is the embedding of the association of value $b_i$ with column $a_i$, i.e. vector \texttt{tt\_ai\_bi}; $w_i$ is a weight calculated as the inverse frequency of value $b_i$ in column $a_i$ in the input table, scaled to $[0.1,1.0]$; and $\alpha$ takes value $1$ for equality predicates and $-1$ for non-equality predicates.
    \item For attribute projection, given an attribute name $a$ and a row embedding $\varphi(r)$, this vector is directly used as the candidate in equation~\eqref{eq:projection}, i.e. the projected value is 
    $$\hat{b}=\argmax_{b\in\mathcal{B}(a)} \{\text{sim}(\varphi(r), \varphi(a,b))\}$$
    where $\mathcal{B}(a)$ is the set of values that appear in column $a$.
\end{compactitem}

\section{Hyperdimensional computing}
\label{sec:hdc}

\textit{Hyperdimensional computing} (HDC), also known as \textit{Vector Symbolic Architectures} (VSA), is a computational framework that represents and processes data using high-dimensional vectors~\cite{kleyko2022surveyi}. The core idea is to encode information into \textit{distributed representations}, leveraging the properties of high-dimensional spaces to enable efficient and robust computation. HDC has been applied in various domains, including cognitive modeling, machine learning, and classification, gaining attention as an alternative to the traditional computing paradigm~\cite{kleyko2023surveyii}. In addition to its efficient use of resources, low latency and robustness to noise, HDC offers tools to establish theoretical guarantees for the performance of its operations~\cite{Thomas2021perspective}.

Specific HDC models differ in how they map data to vectors and in the operations they use to manipulate vectors. A HDC model consist of a mapping $\varphi:A\rightarrow \mathcal{H}$ from a set $A$ of atomic symbols to a $d$-dimensional vector space, producing \textit{atomic vectors} that are stored in a dictionary-like structure called \textit{item memory} or \textit{codebook}.

Additionally, the model defines the following operations:
\begin{itemize}
    \item \textit{Similarity} $\text{sim}(\mathbf{x}, \mathbf{y})$ measures how close vectors $\mathbf{x}$ and $\mathbf{y}$ are. Based on the possible range $[l,h]$ of similarity values, we say that vectors are \textit{similar} if they have high similarity values, \textit{opposite} for low values, and \textit{pseudo-orthogonal} or \textit{dissimilar} for values around the middle of the range.
    \item \textit{Bundling} $\oplus:\mathcal{H}^m\rightarrow \mathcal{H}$ compiles $m$ vectors into a new vector $\bigoplus_{i=1}^m\mathbf{x}_i$ that is similar to every $\mathbf{x}_i$.
    \item \textit{Binding} $\otimes:\mathcal{H}\times\mathcal{H}\rightarrow \mathcal{H}$ generates $\mathbf{x}\otimes\mathbf{y}$ that is pseudo-orthogonal to both $\mathbf{x}$ and $\mathbf{y}$. 
    \item \textit{Unbinding} $\otimes^{-1}:\mathcal{H}\times\mathcal{H}\rightarrow \mathcal{H}$ acts as the inverse of binding, being defined in terms of the binding operator and an inverse element $\mathbf{x}^{-1}$ such that if $\mathbf{z}=\mathbf{x}\otimes \mathbf{y}$, then $$\otimes^{-1}(\mathbf{z},\mathbf{x})=\mathbf{z}\otimes \mathbf{x}^{-1}=\mathbf{y}$$
\end{itemize}
By composing these operations, it is possible to represent data structures such as sets, sequences and trees over atomic symbols in vector space~\cite{Thomas2021perspective,frady2020resonator}. The resulting \textit{composite vectors} are then stored in an \textit{associative memory}, on top of which similarity search can be performed to retrieve similar vectors, with the objective of identifying similar data structures.

HDC models employ randomization to enable reliable similarity-based retrieval also of composite vectors: when the atomic vectors $\varphi$ are drawn at random, they are pairwise pseudo-orthogonal with high probability. The HDC operations exploit this property to generate composite vectors that preserve pseudo-orthogonality (resp. similarity, when desired), hence allowing reliable similarity-based retrieval also of composite vectors.
The dimensionality of $\mathcal{H}$ is important in this respect, as higher dimensions lead to a reduction in \textit{cross-talk noise} when decoding. We next illustrate the HDC modus operandi on the embedding of tabular rows.

\subsection{HRR model}\label{sec:hrr-model}

In this paper we adopt the \textit{Holographic Reduced Representations} (HRR) \cite{plate1995holographic}  HDC model. Like EmbDI and other common neural (non-HDC) embeddings, but in contrast to other popular HDC models such as Binary Splatter Codes, HRR produce real-valued vectors.
Concretely, HRR uses $d$-dimensional real-valued vectors constructed either by (1) sampling the entries from the normal distribution $\mathbf{x}[k]\sim\mathcal{N}(0,1/d)$, for $0\leq k\leq d-1$, or (2) sampling from the $d$-dimensional unit sphere $\mathbf{x}\sim\text{Unif}(S^{d-1})$, where $S^{d-1}=\{\mathbf{x}\in\mathbb{R}^d\mid \|\mathbf{x}\|=1\}$. We use the second approach, as it guarantees that similarity values are always within the range $[-1,1]$, and vectors are normalized, having unit vector norm. In the first approach this only holds in expectation~\cite{plate2003holographic}.
\begin{compactitem}
    \item Similarity is the cosine similarity
    $$\text{sim}(\mathbf{x}, \mathbf{y}) = \frac{\langle\mathbf{x}, \mathbf{y}\rangle}{\|\mathbf{x}\|\cdot\|\mathbf{y}\|}$$
    where $\langle\cdot,\cdot\rangle$ is the dot product in $\reals^d$. Vectors are pseudo-orthogonal if $\text{sim}(\mathbf{x}, \mathbf{y}) \approx 0$.     \item Bundling is the normalized element-wise sum, i.e. 
    $$\left(\bigoplus_{i=1}^m\mathbf{x}_i\right)[k] = \sum_{i=1}^m\mathbf{x}_i[k],\quad 0\leq k\leq d-1$$
    \item Binding is the circular convolution $\circledast$. For $0\leq k\leq d-1$,
    \begin{align*}
        (\mathbf{x}\otimes \mathbf{y})[k]&= (\mathbf{x}\circledast\mathbf{y})[k]\\
        &=\sum_{n=0}^{d-1}\mathbf{x}[n]\,\mathbf{y}[(k-n)\,\text{mod}\, d]
    \end{align*}
    \item Unbinding is performed by using the pseudo-inverse or \textit{involution} of $\mathbf{x}$, given by $\mathbf{x}^{-1}=(\mathbf{x}[0],\mathbf{x}[d-1],\ldots,\mathbf{x}[1])$. It satisfies $\mathbf{y}\approx(\mathbf{x}\circledast\mathbf{y})\circledast\mathbf{x}^{-1}$.
\end{compactitem}

Since we work with atomic vectors sampled from the unit sphere, we add a normalization step after bundling and binding to ensure the resulting vectors have unit length. Thus, $\text{sim}(\mathbf{x}, \mathbf{y})=\langle\mathbf{x}, \mathbf{y}\rangle$ and the similarity values are always in the range $[-1,1]$. In particular, vectors are similar if $\langle\mathbf{x}, \mathbf{y}\rangle\approx 1$ and opposite if $\langle\mathbf{x}, \mathbf{y}\rangle\approx -1$.

\subsection{Encoding and decoding in HRR}
\label{sec:encod-decod-hrr}

Considering the presentation of record encodings proposed by Kanerva~\cite{kanerva2009hyperdimensional}, we use an equivalent notion for rows in our setting. Given a row $r=(a_1:b_1, \ldots, a_m:b_m)$, its encoding is given by:
\begin{equation}\label{eq:row-encoding}
    \mathbf{r}=\text{norm}\left(\sum_{i=1}^m \text{norm}(\mathbf{a}_i\circledast \mathbf{b}_i)\right)
\end{equation}
where $\text{norm}(\mathbf{x})=\mathbf{x}/\|\mathbf{x}\|$ and $\mathbf{a}_i,\mathbf{b}_i\sim\text{Unif}(S^{d-1})$ are atomic vectors representing attributes and values. With this,
\begin{compactitem}[--]
    \item Row retrieval is performed by similarity search using the encoding of selection predicates as probes over the associative memory containing the row encodings. For an equality predicate, we encode it using equation~\eqref{eq:row-encoding}. Conversely, for a non-equality predicate $q=(a_1:b_1, \ldots, a_n:b_n)$, we use
    \begin{equation*}
        \mathbf{q}=\text{norm}\left(\sum_{i=1}^n \text{norm}(\mathbf{a}_i\circledast -\mathbf{b}_i)\right)
    \end{equation*}
    to search for dissimilar values to the ones in the predicate.
    \item Attribute projection is performed by unbinding the row encoding with the attribute name, i.e. for row $r$ and attribute $a$, the candidate vector is given by $\text{norm}(\mathbf{r}\circledast \mathbf{a}^{-1})$, thus
    $$\hat{b}=\argmax_{b\in\mathcal{B}} \{\text{sim}(\mathbf{r}\circledast \mathbf{a}^{-1},\varphi(b))\}$$
    Notice that, in contrast with EmbDI, we search over the full set of values $\mathcal{B}$.
\end{compactitem}

\subsection{Thresholds for row retrieval}\label{sec:thresholds}

We derive the following description of the expected similarity between selection
predicates and rows, in order to  correctly discriminate between positive
and negative matches. For a row $r$ with $m$ attributes and an equality
predicate $q$ with $n$ attributes (all of which are also attributes of $r$), we say they have an
\textit{equality match} if $r(a_i)=b_i$ for every $(a_i:b_i)$ in
$q$. Analogously, for a non-equality predicate $q$, we say that $r$ and $q$ have
a \textit{non-equality match} if $r(a_i)\neq b_i$ for every $(a_i:b_i)$ in $q$.

\begin{proposition}\label{prop:expectation-equality}
    For $1\leq n\leq m$, let $r=(a_1:b_1, \ldots, a_m:b_m)$ be a row and $q=(a_{k_1}:v_1, \ldots, a_{k_n}:v_n)$ be an equality predicate such that $\{k_1,\ldots,k_n\}\subseteq\{1,\ldots,m\}$. Let $\mathbf{r}$ and $\mathbf{q}$ be the $d$-dimensional HRR encodings from Eq.~\eqref{eq:row-encoding} for $r$ and $q$ respectively.  Then, as $d\rightarrow\infty$,
    \begin{align*}
        \mathbb{E}[\langle\mathbf{r},\mathbf{q}\rangle|\ \text{equality match}]&=\sqrt{\dfrac{n}{m}},\\
        \text{Var}[\langle\mathbf{r},\mathbf{q}\rangle|\ \text{equality match}]&\leq\frac{8dm-8d+8m+4dn+n-8}{4md^2}.
    \end{align*} 
    Here, expectation and variance are taken over the random choices of the atomic vectors.
\end{proposition}
Proposition~\ref{prop:expectation-equality} shows that the expected similarity and its variance in the row retrieval task are independent of the contents of $r$ and $q$, depending only on their sizes and, in the case of the variance, on the dimension $d$. 

In the literature, expectation and variance have been studied for analogous but distinct settings~\cite{plate2003holographic}. Specifically, the known analyses (1) adopt the HRR model where vectors have components sampled from $\mathcal{N}(0,1/d)$, and (2) consider different decoding tasks to ours: similarity between two bound pairs of atomic vectors, or similarity of a vector against bundle, where all vectors involved are either atomic or non-independent linear combinations built from a fixed common vector (presented as a superposition memory with similarity among vectors in~\cite{plate2003holographic}). Our contribution lies in the analysis of the normalized unit sphere setting, focusing on decoding of rows, where bundles can be of arbitrary size and are composed of bound pairs of atomic vectors. In this sense, we describe the behavior of a more complex encoding scheme that also benefits from the interpretability of similarity values given by unit sphere normalization.

As a complement to the equality case, an analogous result is obtained for non-equality predicates. The proofs for both results are given in the Appendix.

\begin{proposition}\label{prop:expectation-non-equality}
    For $1\leq n\leq m$, let $r$ be a row and $q$ be a non-equality predicate that satisfy the hypotheses of Proposition~\ref{prop:expectation-equality}. Then, as $d\rightarrow\infty$, their encodings satisfy
    \begin{align*}
        \mathbb{E}[\langle\mathbf{r},\mathbf{q}\rangle|\ \text{non-equality match}]&=0,\\
        \text{Var}[\langle\mathbf{r},\mathbf{q}\rangle|\ \text{non-equality match}]&\leq \frac{2d+2}{d^2}.
    \end{align*}
\end{proposition}

Based on Propositions~\ref{prop:expectation-equality} and \ref{prop:expectation-non-equality}, we define the following thresholds for the problem of row retrieval for $m$-attribute rows with $n$-attribute predicates.

\begin{definition}\label{def:thresholds}
\it    Given $1\leq n\leq m$, we define the threshold for $m$-attribute rows and $n$-attribute equality and non-equality predicates, for $d$-dimensional HRR encodings, as
    \begin{align*}
        \tau_{\text{eq}}(n,m,d) &= \sqrt{\dfrac{n}{m}}-\sqrt{\frac{8dm-8d+8m+4dn+n-8}{4md^2}}\\
        \tau_{\text{neq}}(d)&=-\sqrt{\frac{2d+2}{d^2}}.
    \end{align*}
  \end{definition}

The advised thresholds from Def.~\ref{def:thresholds} consider one standard deviation from the expected value. In this way, we expand the range of similarity values that we take as a match (all $\tau\geq\tau_{\text{eq}}(n,m,d)$ or $\tau\geq\tau_{\text{neq}}(d)$, accordingly) in order to address the interference caused by cross-talk noise between bound pairs in bundles. As suggested by the variances, this cross-talk effect is reduced as $d$ increases, so that the proposed thresholds converge to their respective expected values $\sqrt{\frac{n}{m}}$ and 0, as $d\rightarrow\infty$.

Because the attribute projection task is performed by searching for the most similar value vector to the unbound candidate $\mathbf{r}\circledast\mathbf{a}^{-1}$, no threshold is involved in the search and an equivalent analysis to the one carried out for row retrieval is not adequate. Instead, the probability of correct decoding for attribute projection could be studied. This is left for future work.

\section{Experimental setting}
\label{sec:experimental-setting}

\myparagraph{Datasets.} We generate our tables from the two real-world tabular datasets used in~\cite{mellouli2025selection} for the evaluation of EmbDI. Both were preprocessed to contain only one table per dataset.
\begin{itemize}
    \item Movie: $m_{\max}=15$ columns and $49875$ rows
    \item DBLP: $m_{\max}=4$ columns and $66876$ rows
\end{itemize}
Then, we generated tables with $1\leq m\leq m_{\max}$ attributes by considering only the first $m$ columns from the original tables, keeping duplicated rows when they existed. From this, we had a set $\mathcal{T}$ containing 15 tables for the Movie dataset and 4 for DBLP.

\myparagraph{Selection predicates.} For each table $T\in\mathcal{T}$ of $m$ columns, we generated selection predicates by sampling rows from $T$ and taking $1\leq n\leq m$ attribute-value pairs from the row to generate a selection predicate that matched those values. Then, we generated 10 equality and 10 non-equality predicates for each value of $n$ and each table $T$. Additionally, we generated equality predicates with zero matches. For this, we sampled rows, but we modified one attribute to a value present in a different row. For tables with only one column, we sampled values from larger tables. Thus, we had 10 zero-match equality predicates for each $n$ and table $T$, such that they are as close as possible to an existing row in the table.

\myparagraph{Embedding methods.} We compare the performance of EmbDI and HDC for different dimensions. For EmbDI, we took dimensions $d\in\{300, 512\}$, sampling 500,000 random walks for the Movie dataset and 1,000,000 for the DBLP dataset. These values were tested so that almost every required entity in the table has an embedding. We used Word2Vec with window size 3, as in~\cite{mellouli2025selection}. For HDC, we used the HRR model with $d\in\{300, 512, 1024\}$. We executed 3 runs for each dimension on all experiments, to account for the randomness in the generation of the atomic vectors. Results for specific dimensions are always averaged over the 3 runs.

\myparagraph{Evaluation.} For the row retrieval task, we used two approaches. First, we evaluated the F1-score as a function of $k$ when retrieving the top-$k$ closest rows for $k\in\{1,2,5,10,20\}$ for each table $T$ and each selection predicate length $n$. This was done to replicate the evaluation setup from~\cite{mellouli2025selection}. Second, we obtained the F1-score as a function of the threshold $\tau$ when computing the set $S(T,q,\tau)$ for every table $T$ and selection predicate $q$. For equality predicates, $\tau\in[0.1,1]$, while $\tau\in[-0.3,0.2]$ for non-equality predicates, taking values in the specified ranges with a step of $s=0.05$. 

Additionally, we considered the subtask of zero-match predicates. For each HRR dimension $d$, we measured the number of rows retrieved $|S(T,q,\tau)|$ per predicate $q$ and table $T$ when fixing the threshold value $\tau=\tau_{\text{eq}}(n,\text{cols}(T),d)$, where $\text{cols}(T)$ is the number of columns in table $T$. Since EmbDI does not have a single threshold description, we only applied this task to HDC.

Finally, for attribute projection, we sampled 50 rows from each table $T$ and obtain the candidate value for each attribute in every row. We then evaluated performance by measuring the accuracy of exact matches for each table $T$ and attribute name. We executed 3 runs for HRR and averaged the results.

\section{Results}
\label{sec:results}

\subsection{Row retrieval}

\myparagraph{Equality selection predicates.} The overall behavior of both approaches varies with the number of columns $m$ and predicate size~$n$.

\begin{figure}
    \centering
    \includegraphics[width=0.46\textwidth]{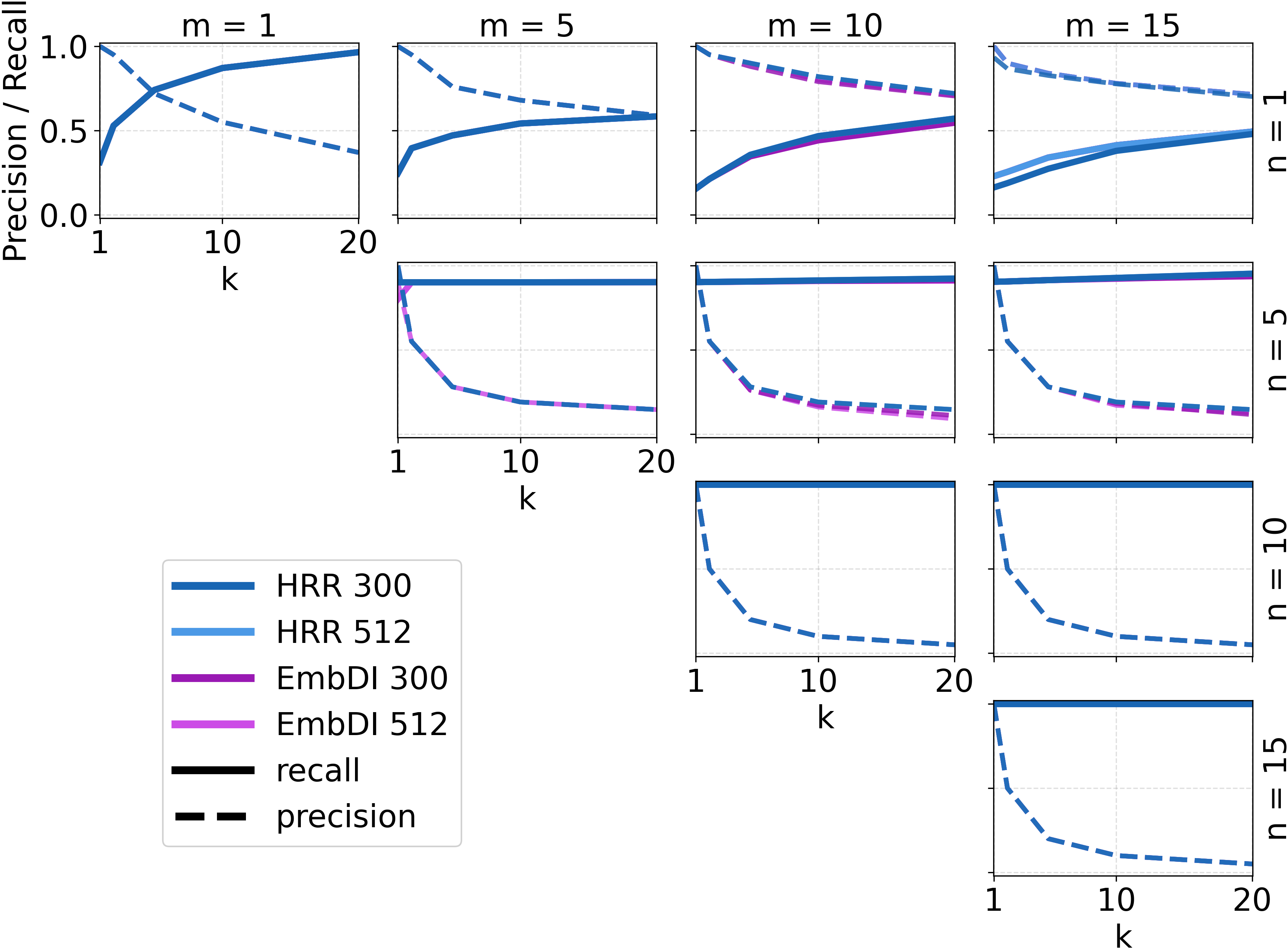}
    \caption{Recall and precision vs top-$k$ for equality predicates on the Movie dataset, for selected values of $m,n$. }
    \label{fig:grid_topk_movie}
\end{figure}

Figure~\ref{fig:grid_topk_movie} shows recall and precision when retrieving the top-$k$ most similar rows on the Movie dataset, for selected values of $m,n$ and for $k\in\{1,2,5,10,20\}$. Both EmbDI and HDC behave very similarly, with almost identical metrics across all $k$ values, explaining why the curves for EmbDI are often not visible in Figure~\ref{fig:grid_topk_movie}. This  behavior  is also present in the full grid available in the Appendix in Figure~\ref{fig:topk_equality_movie}. 
For $k \in \{1,2\}$ and increasing values of $m$ and $n$, both approaches have precision and recall approaching 1.0. They are hence comparable in how they correctly select the most similar row(s) for most combinations of $m$ and $n$. However, as $k$ increases, precision drops rapidly, for all $m,n$. Note that it is impossible to determine the correct value of $k$ to use, as this depends on the number of answers in the query result. As such, top-$k$ based retrieval is in general inadequate for the general row retrieval problem. This is especially true for predicates with zero matches, where the top-$k$ approach will always incorrectly return $k \geq 1$ rows.

\begin{figure*}[t] \centering

\begin{subfigure}[b]{0.22\textwidth}
        \centering
        \includegraphics[width=\linewidth]{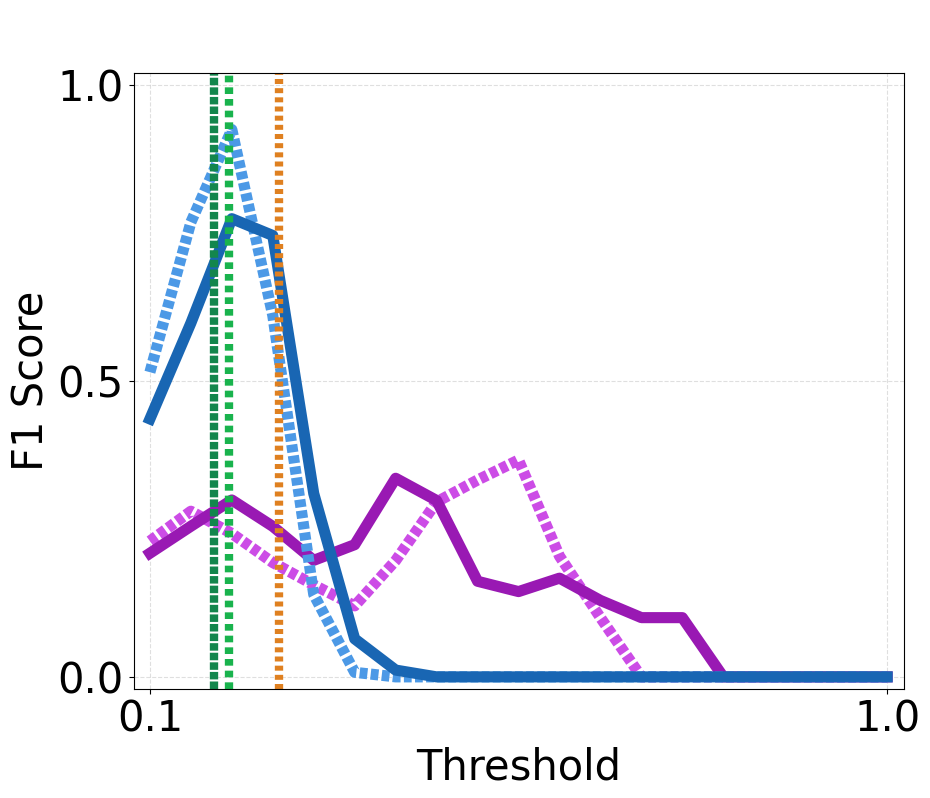}
        \caption{$n=1$ and $m=15$}
        \label{fig:equality_img1}
    \end{subfigure}
    \hfill \begin{subfigure}[b]{0.22\textwidth}
        \centering
        \includegraphics[width=\linewidth]{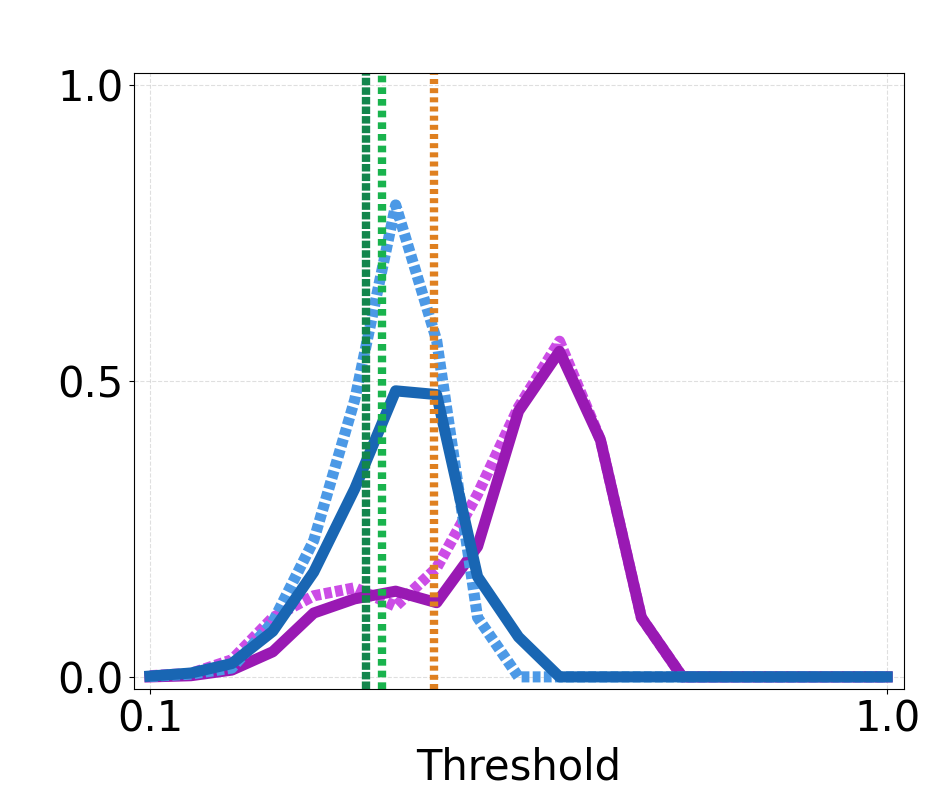}
        \caption{$n=3$ and $m=15$}
        \label{fig:equality_img2}
    \end{subfigure}
    \hfill
\begin{subfigure}[b]{0.22\textwidth}
        \centering
        \includegraphics[width=\linewidth]{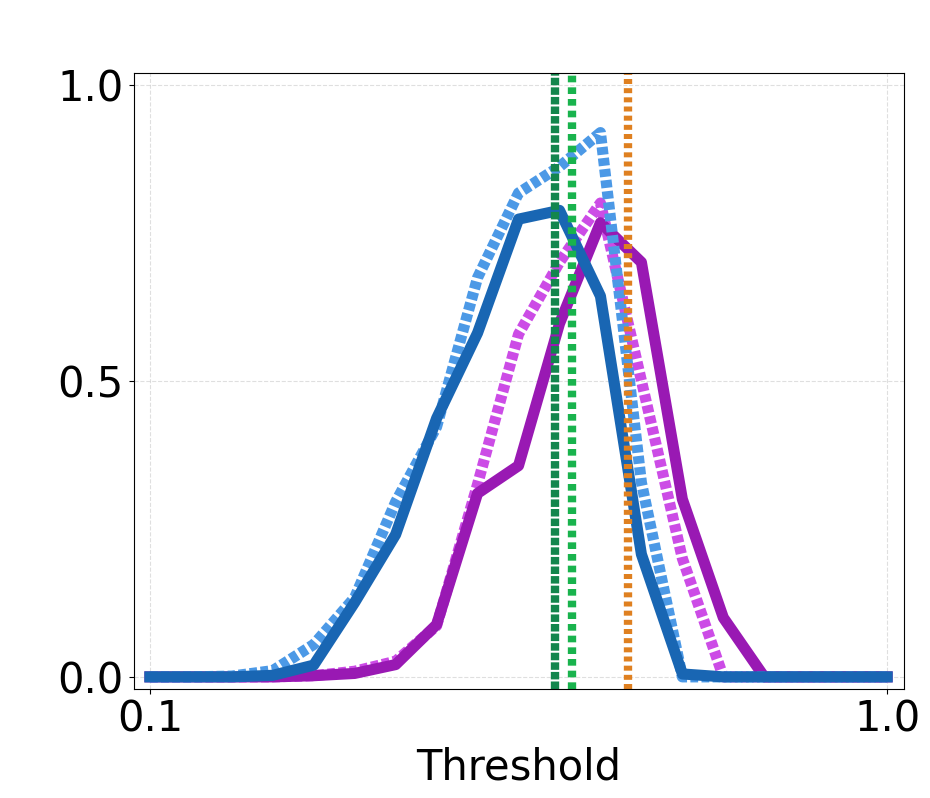}
        \caption{$n=7$ and $m=15$}
        \label{fig:equality_img3}
    \end{subfigure}
    \hfill
\begin{subfigure}[b]{0.22\textwidth}
        \centering
        \includegraphics[width=\linewidth]{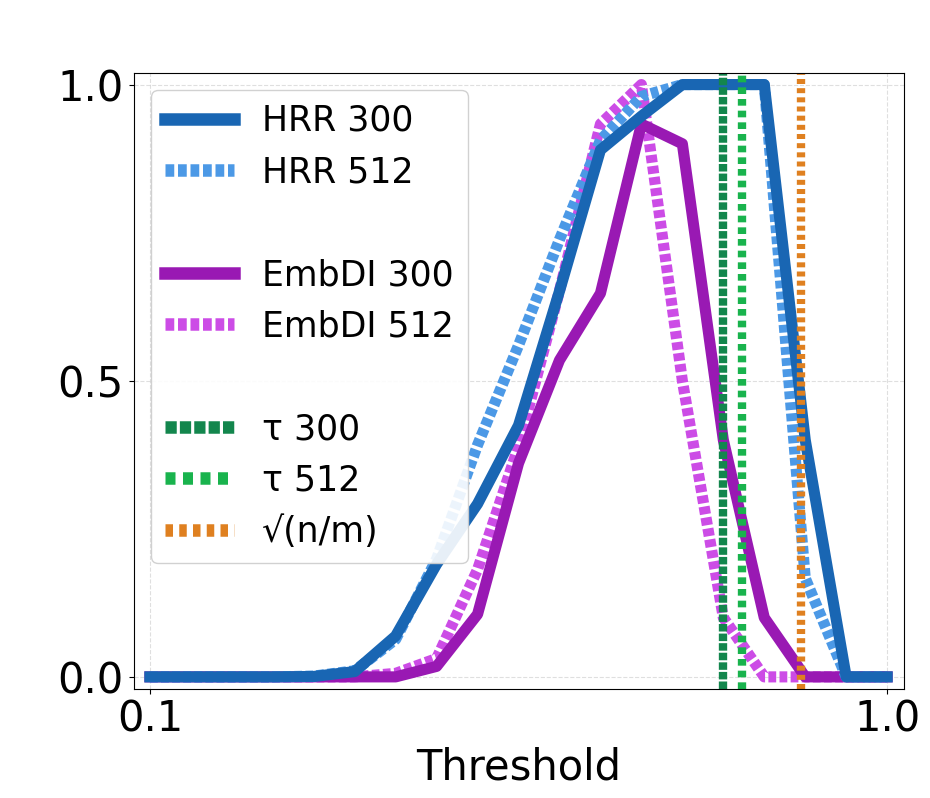}
        \caption{$n=12$ and $m=15$}
        \label{fig:equality_img4}
    \end{subfigure}

    \caption{Selected results of F1-score vs similarity threshold for row retrieval over equality predicates on the Movie dataset. }\label{fig:equality_movie_four_per_row}
\end{figure*} 

\begin{figure}
    \centering
    \includegraphics[width=0.47\textwidth]{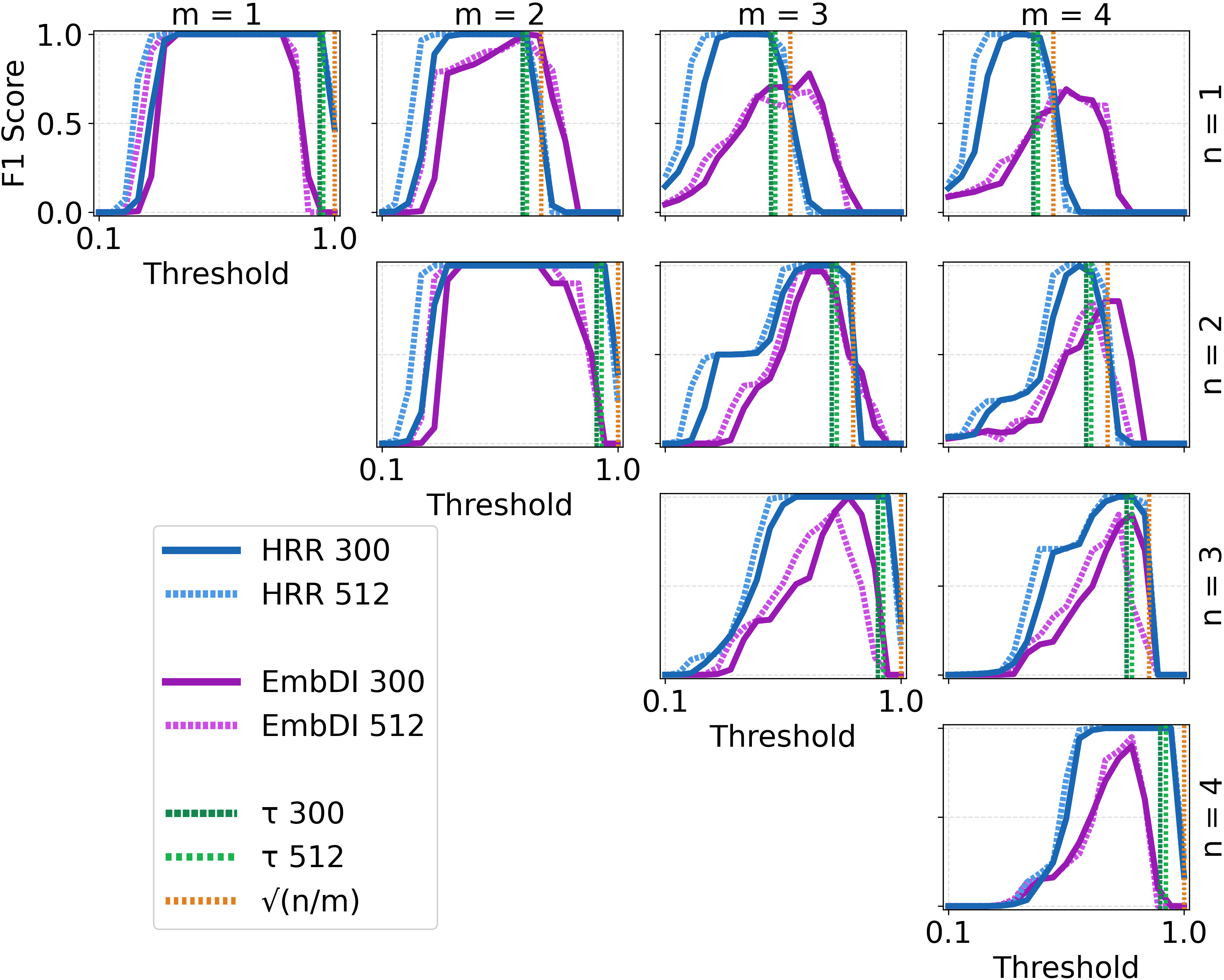}
    \caption{F1-score vs similarity threshold for row retrieval, for the DBLP dataset on equality predicates.}
    \Description{Grid of F1-score versus similarity threshold for DBLP row retrieval under equality predicates across table sizes and predicate lengths.}
    \label{fig:grid_equality_dblp}
\end{figure}

As an alternative to top-$k$ selection, we evaluate the performance of HDC and EmbDI when retrieving rows whose similarity is above a certain threshold $\tau$. Figure~\ref{fig:equality_movie_four_per_row} shows selected results for equality predicates on the Movie dataset, while Figure~\ref{fig:grid_equality_dblp} corresponds to the full results for the DBLP dataset. Figure~\ref{fig:grid_equality_movie} in the Appendix shows the full grid for every combination of $m$ and $n$ for the Movie dataset. In these figures, we plot the obtained F1 score as a function of the threshold used, also indicating the expected (Prop.~\ref{prop:expectation-equality}) and advised threshold (Def.~\ref{def:thresholds}).
Across all scenarios, we confirm that dimension has a positive effect on the performance of HDC, with higher $d$ leading to higher F1-scores across different threshold values, explained by the increased robustness to cross-talk in bundles. The performance of EmbDI is not greatly affected by dimension.

For a fixed table size $m$, both HDC and EmbDI tend to increase their performance, measured as peak F1-score, as we increase the number of attributes $n$ in the predicate from 1 to $m$, being more noticeable for larger $m$ values. It is worth noting that for shorter predicates for fixed $m$, HDC has a slight decrease in performance before reaching perfect retrieval for larger $n$ values. EmbDI has a more consistent increase in performance with predicate length $n$.

For smaller fixed values of $n$, the performance of EmbDI degrades as we add more columns to the table. HDC tends to maintain F1-score peaks but the range of threshold values to reach them becomes narrower, particularly for the Movie dataset, which involves longer rows than the DBLP dataset. In particular, $n=1$ is the predicate size for which HDC performs best when compared with EmbDI. For larger fixed values of $n$, both methods exhibit a more stable tendency as we increase $m$ and their performance is comparable. This suggests that rows that are similar to a predicate in HDC tend to have comparable similarity scores across different predicates.

This is confirmed by our theoretical thresholds, that are shown in Figures~\ref{fig:equality_movie_four_per_row} and \ref{fig:grid_equality_dblp} for dimensions $d\in\{300,512,1024\}$, as well as the mean threshold $\tau_{mean}=\sqrt{n/m}$ appearing as the rightmost vertical dashed line. We note that our advised thresholds graphically align with the observed peaks for HDC, showing the displacement of the peaks as we change $m$ and $n$, and suggesting definition~\ref{def:thresholds} is pertinent. Moreover, $\tau_{mean}$ indicates threshold values such that F1-score decreases given the partial exclusion of matching rows.

\begin{table}
    \centering
    \caption{F1-scores for row retrieval on the Movie dataset using $\tau_{\text{eq}}(n,m,d)$ for equality predicates, and average best score for EmbDI across all thresholds. Averaged over values of $n$.}
    \label{tab:threshold_results}
    \setlength{\tabcolsep}{4pt}
    {\footnotesize
    \begin{tabular}{l|rrr|rr}
        \toprule
        $m$ & HRR 300 & HRR 512 & HRR 1024 & EmbDI 300 & EmbDI 512 \\
        \midrule
        1 & 1.00 & 1.00 & 1.00 & 1.00 & 1.00 \\
        2 & 1.00 & 1.00 & 1.00 & 0.92 & 0.95 \\
        3 & 1.00 & 1.00 & 1.00 & 0.99 & 0.98 \\
        4 & 0.99 & 0.99 & 1.00 & 0.89 & 0.88 \\
        5 & 0.99 & 0.98 & 1.00 & 0.78 & 0.83 \\
        6 & 0.97 & 0.99 & 0.99 & 0.85 & 0.83 \\
        7 & 0.94 & 0.99 & 1.00 & 0.79 & 0.79 \\
        8 & 0.90 & 0.99 & 0.99 & 0.71 & 0.75 \\
        9 & 0.90 & 0.97 & 0.99 & 0.66 & 0.65 \\
        10 & 0.85 & 0.97 & 0.99 & 0.71 & 0.69 \\
        11 & 0.85 & 0.97 & 0.99 & 0.77 & 0.80 \\
        12 & 0.84 & 0.93 & 0.99 & 0.72 & 0.74 \\
        13 & 0.85 & 0.92 & 0.98 & 0.79 & 0.79 \\
        14 & 0.85 & 0.92 & 0.98 & 0.79 & 0.79 \\
        15 & 0.82 & 0.92 & 0.99 & 0.76 & 0.82 \\
        \bottomrule
    \end{tabular}
    }
\end{table}

To better illustrate the performance using the proposed thresholds, Table~\ref{tab:threshold_results} contains the F1-scores for each dimension as a function of $m$, averaged over values of $n$ for the Movie dataset. For comparison, we include the average over $n$ of the maximum performance for EmbDI on the same scenario given the threshold values $0.2\leq \tau\leq 1$. We can see that the proposed thresholds for HDC are consistent with the observed peaks, and that they outperform the average best performance of EmbDI in almost every case.

\begin{table}
    \centering
    \caption{Average number of retrieved rows for zero-match equality predicates on the Movie dataset, using $\tau_{\text{eq}}(n,m,d)$.}
    \label{tab:zero_results}
    \setlength{\tabcolsep}{4pt}
    {\footnotesize
    \begin{tabular}{l|rrr}
        \toprule
        $m$ & HRR 300 & HRR 512 & HRR 1024 \\
        \midrule
        1 & 0.00 & 0.00 & 0.00 \\
        2 & 0.00 & 0.00 & 0.00 \\
        3 & 0.00 & 0.00 & 0.00 \\
        4 & 0.00 & 0.00 & 0.00 \\
        5 & 0.00 & 0.00 & 0.00 \\
        6 & 0.01 & 0.00 & 0.00 \\
        7 & 0.37 & 0.00 & 0.00 \\
        8 & 4.04 & 0.00 & 0.00 \\
        9 & 43.04 & 0.14 & 0.00 \\
        10 & 20.60 & 0.66 & 0.00 \\
        11 & 7.20 & 0.29 & 0.00 \\
        12 & 19.20 & 16.44 & 0.00 \\
        13 & 30.83 & 4.43 & 0.01 \\
        14 & 61.42 & 2.85 & 0.06 \\
        15 & 101.99 & 10.95 & 0.40 \\
        \bottomrule
    \end{tabular}
    }
\end{table}

Lastly, to highlight the flexibility of the thresholded approach, we explore the subtask of selecting rows for zero-match predicates. Table~\ref{tab:zero_results} shows the average number of retrieved rows for zero-match equality predicates on the Movie dataset, when using the proposed thresholds. We can see that for HDC, the proposed thresholds are effective at retrieving zero rows for most cases, with the expected degradation for larger $m$ caused by the cross-talk of more complex bundles. This goes in line with the increased performance for larger dimension $d$. Still, for the different values of $d$, the number of retrieved rows is very low compared to the 49875 total rows of the table. For EmbDI, it is not clear how to identify zero-match queries.

\myparagraph{Non-equality selection predicates.} In contrast with equality predicates, trends in retrieval performance using non-equality conditions do not vary much when changing $n$ or $m$. Figure~\ref{fig:nonequality_movie_four_per_row} shows selected examples for the Movie dataset, while Figures~\ref{fig:grid_nonequality_movie} and \ref{fig:grid_nonequality_dblp} contain the full grids for both datasets. Except for short predicates in the Movie dataset, EmbDI has very low performance, while HDC has consistent performance across all $m$ and $n$ values, with results even comparable to the case of equality predicates. Both methods are comparable for $m,n<4$, being able to retrieve with high F1-score given that most rows are dissimilar to the predicate. Therefore, for a small enough threshold, most retrieved rows are correct.  From all, for larger table and predicate sizes, HDC achieves better retrieval for rows being different from specified values, suggesting a better encoding of the notion of being opposite.

\begin{figure*}[t] \centering

\begin{subfigure}[b]{0.22\textwidth}
        \centering
        \includegraphics[width=\linewidth]{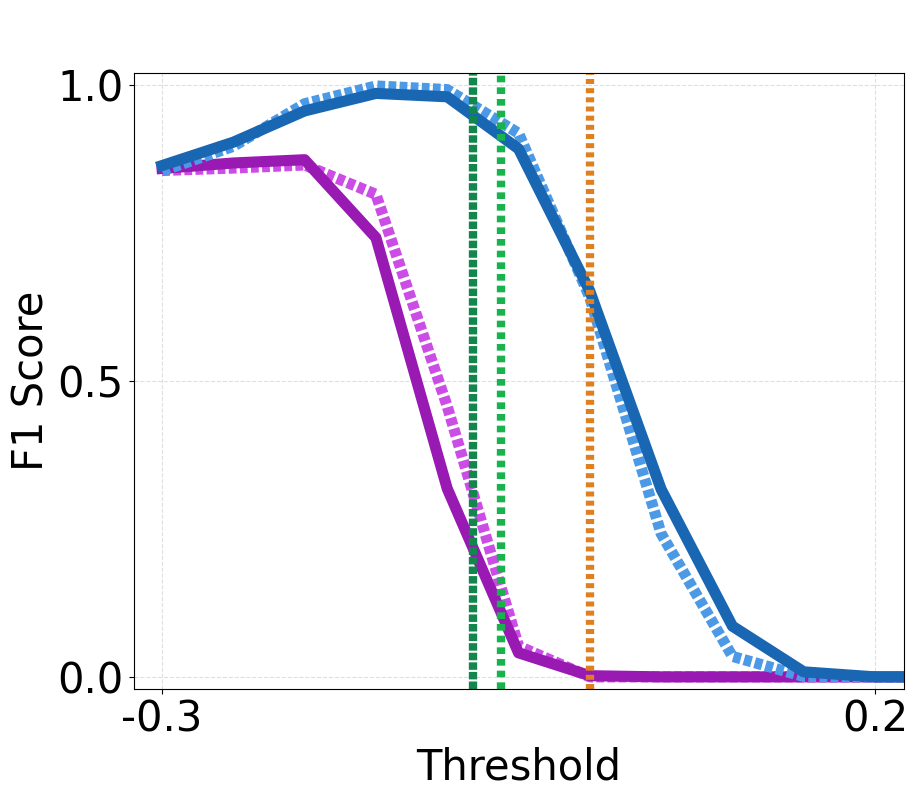}
        \caption{$n=1$ and $m=15$}
        \label{fig:img1}
    \end{subfigure}
    \hfill \begin{subfigure}[b]{0.22\textwidth}
        \centering
        \includegraphics[width=\linewidth]{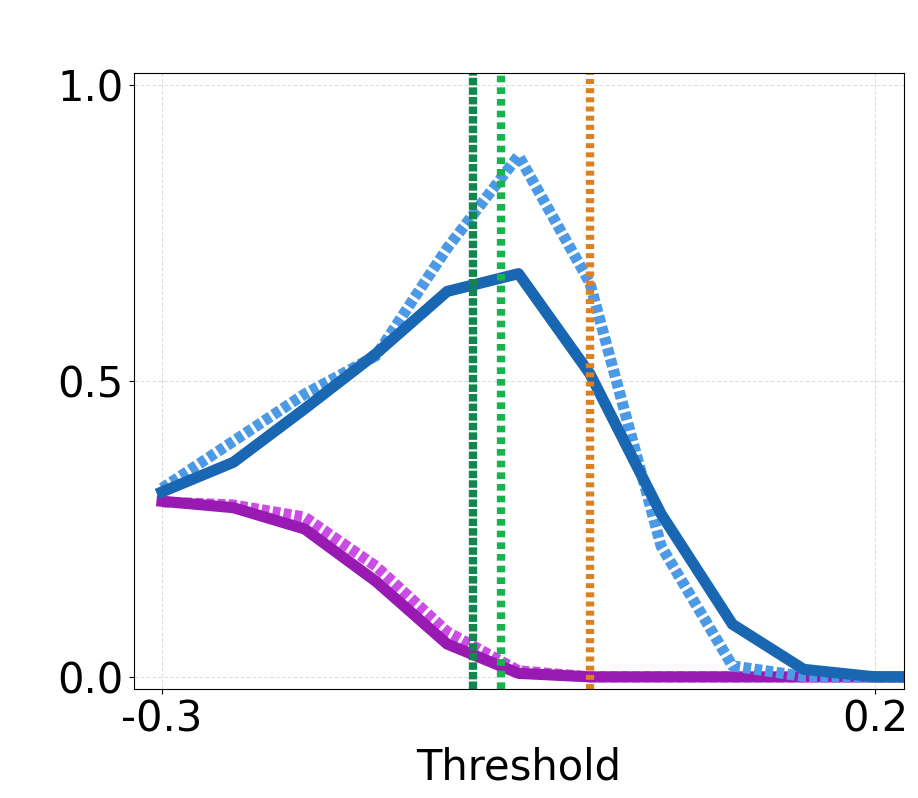}
        \caption{$n=3$ and $m=15$}
        \label{fig:img2}
    \end{subfigure}
    \hfill
\begin{subfigure}[b]{0.22\textwidth}
        \centering
        \includegraphics[width=\linewidth]{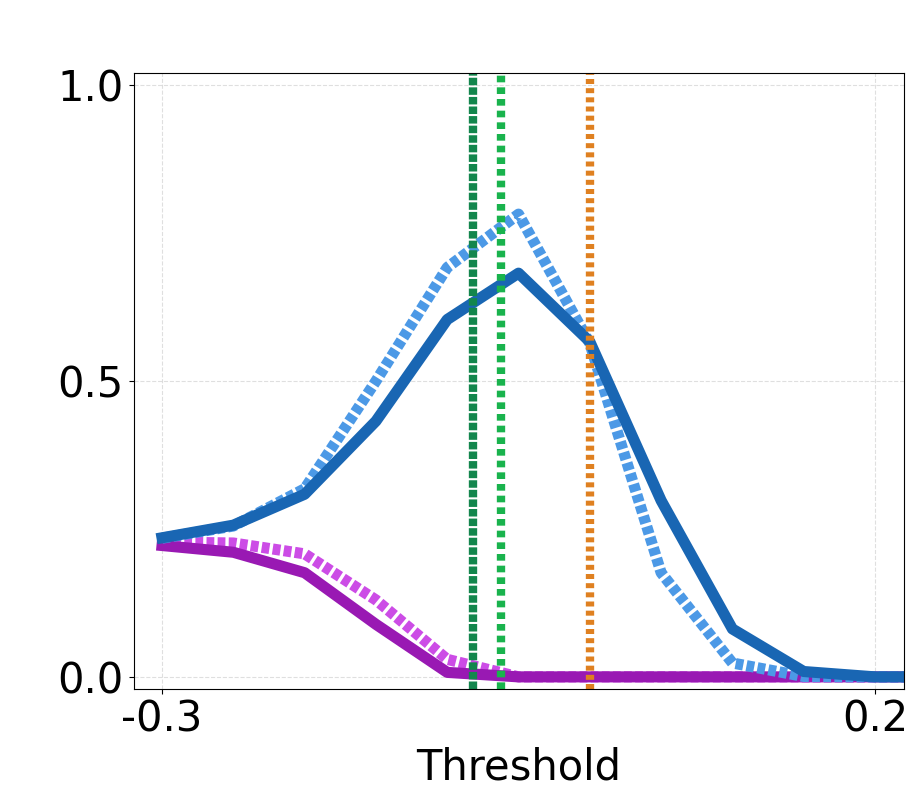}
        \caption{$n=7$ and $m=15$}
        \label{fig:img3}
    \end{subfigure}
    \hfill
\begin{subfigure}[b]{0.22\textwidth}
        \centering
        \includegraphics[width=\linewidth]{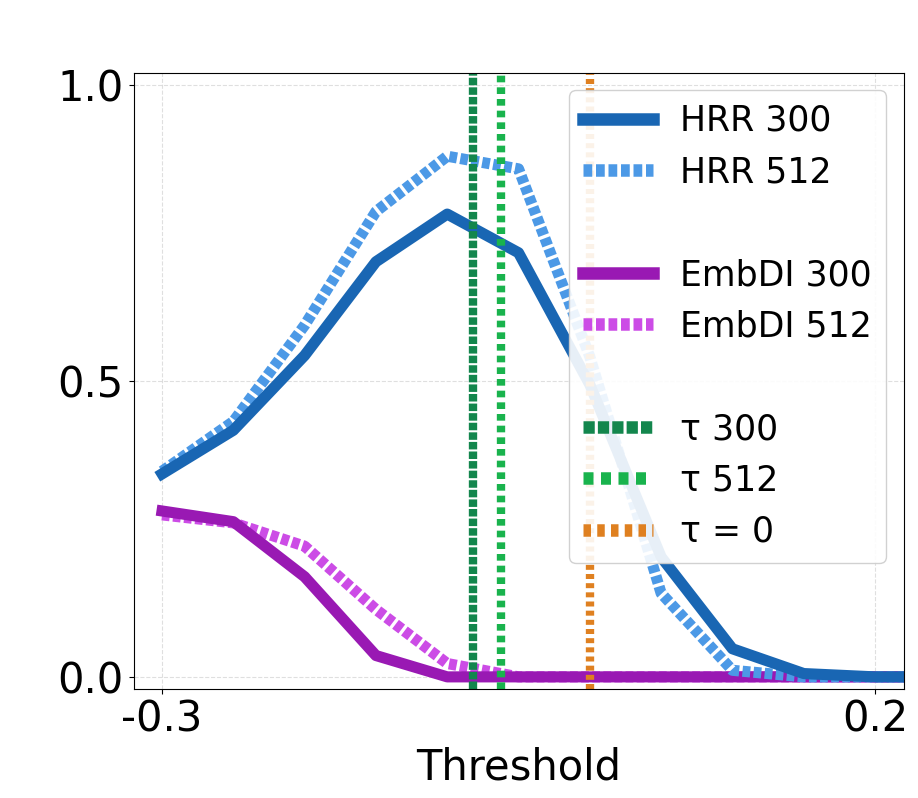}
        \caption{$n=15$ and $m=15$}
        \label{fig:img4}
    \end{subfigure}

    \caption{Selected results of F1-score vs similarity threshold for row retrieval over non-equality predicates on the Movie dataset. Results are shown for fixed table size $m=15$ and increasing predicate length $n$.}
    \label{fig:nonequality_movie_four_per_row}
\end{figure*}

As seen in Figure~\ref{fig:nonequality_movie_four_per_row}, in the same way as in equality predicates, the proposed thresholds for perfect non-equality matching are consistent with the observed peaks for HDC. Furthermore, the mean threshold $\tau_{mean}=0$ correctly indicates the values of threshold such that recall decreases. Finally, we see a slight increase in performance with dimension for HDC, while EmbDI is again not affected by it.

\subsection{Attribute projection}

\begin{table*}[t]
    \centering
    \caption{Average accuracy for attribute projection on full tables}\vspace{-1ex}
    \label{tab:projection_results}
    
\begin{subtable}[b]{0.55\textwidth}
        \centering
        \caption{Movie Dataset (selected columns from $m=15$)}
        \label{tab:table_a}
        {\footnotesize
        \begin{tabular}{l|r|rrrr} 
            \toprule
            model & overall & director &  prod. companies & rating & genres \\
            \midrule
            EmbDI 300 & 0.78 & 0.96 & 0.68 & 0.30 & 0.82 \\
            EmbDI 512 & 0.69 & 0.96 & 0.62 & 0.08 & 0.74 \\
            \midrule
            HRR 300 & 0.50 & 0.42 & 0.51 & 0.57 & 0.48 \\
            HRR 512 & 0.93 & 0.93 & 0.92 & 0.98 & 0.97 \\
            HRR 1024 & 1.00 & 1.00 & 1.00 & 1.00 & 1.00 \\
            \bottomrule
        \end{tabular}
        }
    \end{subtable}
    \hfill \begin{subtable}[b]{0.42\textwidth}
        \centering
        \caption{DBLP dataset ($m=4$)}
        \label{tab:table_b}
        {\footnotesize
        \begin{tabular}{l|r|rrrr}
            \toprule
            model & overall & authors & title & venue & year \\
            \midrule
            EmbDI 300 & 0.82 & 1.00 & 0.98 & 0.70 & 0.62 \\
            EmbDI 512 & 0.80 & 1.00 & 0.98 & 0.64 & 0.56 \\
            \midrule
            HRR 300 & 1.00 & 1.00 & 1.00 & 1.00 & 1.00 \\
            HRR 512 & 1.00 & 1.00 & 1.00 & 1.00 & 1.00 \\
            HRR 1024 & 1.00 & 1.00 & 1.00 & 1.00 & 1.00 \\
            \bottomrule
        \end{tabular}
        }
    \end{subtable}
\end{table*}

Table~\ref{tab:projection_results} shows the results for attribute projection on both datasets. For the Movie dataset, we include a selection of columns exhibiting different attribute cardinalities that have an impact on EmbDI's performance, e.g. \textit{director} has a wide range of values, whereas \textit{rating} has limited numerical values. The table shows the overall accuracy for each model, as well as the accuracy for each column separately.

For both datasets, a higher dimension for EmbDI does not translate into higher accuracy, while for HDC it allows reaching perfect decoding for all columns. The size of the table has an impact on the lower dimensions of HDC, namely $d=300$ for the larger table of the Movie dataset, which gives HDC lower performance than its neural counterpart with the same dimension. Nonetheless, $d=512$ is sufficient to overcome this issue in both datasets. Additionally, the performance of HDC is more consistent across columns, while for EmbDI it varies greatly as was already noted in~\cite{mellouli2025selection}. In particular, EmbDI is noticeably weak at decoding low cardinality columns.

\section{Conclusions and future work}\label{sec:discussion}

The experiments presented show that HDC offers a framework for performing simple similarity-based querying on tabular data embeddings, both for row retrieval and attribute projection. Dimension is a key factor in the implementation of HDC, with higher dimensions outperforming EmbDI for both tasks. Even for comparable dimensions, we observe that HDC can generalize over attributes for the projection task, and reach higher performance for row retrieval when using advised thresholds. If we consider even higher dimensions, HDC can achieve perfect decoding for projection, however, vectors that are too large may not be desirable in practice.

The most important strength of HDC in this context is the threshold description for general selection predicates. Retrieving results based on a fixed $k$ for nearest neighbor search is insufficient, as it impacts precision and recall directly whenever the number of results for an arbitrary predicate is different from $k$. Moreover, it is not even possible to correctly identify predicates with zero results. Our threshold definitions show that it is possible to achieve good performance across different table and predicate lengths, a tool that is not available for EmbDI. In this sense, the principled framework of operations and encoding in HDC is crucial to simulate exact matching in a similarity-based way.

Interestingly, the previous observations apply for the problem of exact matching, both for projection and retrieval, which we argue is not the only environment where HDC shines. The two-task evaluation we present is a method to assess the ability of embeddings to preserve tabular structure, but it does not fully capture other types of querying that can be performed based on similarity. Altogether,  our results for exact matching are a strong indicator that HDC preserves and can use similarity in composite entities based on the similarity of their components, which opens the door to (1) queries that are based on semantic relations between values and (2) alternative structure encodings.

With regard to \textit{semantic queries}, the use of  random atomic vectors is essential for HDC. In many situations, however, we may want to start from pre-trained vectors for atomic data values (e.g. Word2Vec), such that embedded row vectors become similar even if they contain different, but related, data values.  It is an open question to what extent the decoding capabilities of HDC can be coupled with pre-trained atomic embeddings. This motivates the definition of alternative notions of querying that go beyond the exact matching of relational databases, to evaluate embeddings.

We have focused in this paper on row embeddings. The suitability of HDC for other embedding kinds, such as column or table embeddings, is also an open question. In particular, based on the idea of visually traversing tables as an inspiration for a graph representation and embedding in~\cite{tchuitcheu2024table}, we can think on having a hybrid embedding that navigates in the table structure non-uniformly, e.g. by walking over the grid of the table and encoding positional information about cells. In this context, yet another HDC operation, permutation~\cite{frady2020resonator}, is relevant to encode positions of tabular entities.

\begin{acks}
  This work was supported by the Flanders AI (FAIR) research program; by the
  Research Foundation Flanders (FWO) under research project Grant No. G019222N;
  and by the Bijzonder Onderzoeksfonds (BOF) of Hasselt University Grant
  No. BOF20ZAP02.
\end{acks}

\bibliographystyle{ACM-Reference-Format}
\bibliography{others/sample}

\begin{appendices}
\section{Proof of similarity results}

\textbf{On nomenclature.} In the following results, we refer to the circular convolution $\mathbf{c}=\mathbf{a}\circledast\mathbf{b}$ as a \textit{binding} or \textit{bound pair}. We refer to the sum of normalized bound pairs $\mathbf{r}=\sum_{i=1}^m \text{norm}(\mathbf{a}_i\circledast \mathbf{b}_i)$ as a \textit{bundle of normalized bound pairs}. The similarity between two unit vectors $\mathbf{x},\mathbf{y}$ is given by their inner product $\langle\mathbf{x},\mathbf{y}\rangle$.

\begin{lemma}\label{lemma:sphere}
    If $\mathbf{a}=(\mathbf{a}[0],\ldots,\mathbf{a}[d-1])\sim\text{Unif}(S^{d-1})$, then the individual components $\mathbf{a}[i]\sim\mathcal{N}(0,1/d)$ for every $0\leq i\leq d-1$ as $d\rightarrow\infty$.
\end{lemma}
\begin{proof}
    It follows from Theorem 3.3.9 in \cite{Vershynin_2018}.
\end{proof}

\begin{lemma}\label{lemma:inner-product}
    For vectors $\mathbf{a},\mathbf{b},\mathbf{c},\mathbf{d}\sim\text{Unif}(S^{d-1})$, the dot product $X=\langle\mathbf{a}\circledast\mathbf{b},\mathbf{c}\circledast\mathbf{d}\rangle$ satisfies the following properties as $d\rightarrow\infty$:
    \begin{itemize}\itemsep=0.5em
        \item If $\mathbf{a}=\mathbf{c}$, $\mathbf{b}=\mathbf{d}$, and $\mathbf{a},\mathbf{b}$ are independent, then $X\sim\mathcal{N}\left(1,\frac{6d+4}{d^2}\right)$.
        \item If $\mathbf{a}=\mathbf{c}$ and $\mathbf{a},\mathbf{b},\mathbf{d}$ are independent, then $X\sim\mathcal{N}\left(0,\frac{2d+2}{d^2}\right)$.
        \item If $\mathbf{a},\mathbf{b},\mathbf{c},\mathbf{d}$ are independent, then $X\sim\mathcal{N}\left(0,\frac{1}{d}\right)$.
    \end{itemize}
\end{lemma}
\begin{proof}
    Follows from Lemma~\ref{lemma:sphere} and rows (6), (8) and (10) of Table 3 in \cite{plate2003holographic}.
\end{proof}

\begin{lemma}\label{lemma:convolution}
    For $\mathbf{a},\mathbf{b}\sim\text{Unif}(S^{d-1})$ independent, the coordinates of their circular convolution $\mathbf{c}=\mathbf{a}\circledast\mathbf{b}$ are independent and distributed as $\mathcal{N}(0,1/d)$ for $d\rightarrow\infty$.
\end{lemma}
\begin{proof}
    An arbitrary coordinate of the binding $\mathbf{c}=\mathbf{a}\circledast\mathbf{b}$ are given by
    $$\mathbf{c}[i]=\sum_{k=0}^{d-1} \mathbf{a}[k]\ \mathbf{b}[(i-k)\bmod d],$$
    for $0\leq i\leq d-1$. Let $X_k=\mathbf{a}[k]$ and $Y_k=\mathbf{b}[(i-k)\bmod d]$ so that 
    $$\mathbf{c}[i]=\sum_{k=0}^{d-1}X_kY_k.$$
    Since $\mathbf{a}$ and $\mathbf{b}$ are independent, the pairs $(X_k,Y_k)$ are independent for $0\leq k\leq d-1$. By Lemma~\ref{lemma:sphere} we know each $X_k, Y_k\sim\mathcal{N}(0,1/d)$  as $d\rightarrow\infty$. Noticing each product can be written as
    $$X_kY_k=\frac{1}{4}\left((X_k+Y_k)^2-(X_k-Y_k)^2\right),$$
    let $U_k=X_k+Y_k$ and $V_k=X_k-Y_k$. Then $U_k,V_k\sim\mathcal{N}(0,2/d)$ and $U_k^2,V_k^2\sim\frac{2}{d}\chi^2(1)$. With this, the sum of $d$ independent products can be written as
    $$\mathbf{c}[i]=\sum_{k=0}^{d-1} X_kY_k=\left(\sum_{k=0}^{d-1}\frac{1}{4}U_k^2\right)-\left(\sum_{k=0}^{d-1}\frac{1}{4}V_k^2\right)=U-V,$$
    where $U$ and $V$ are sums of $d$ independent $\frac{1}{2d}\chi^2(1)$ variables, thus $U,V\sim \frac{1}{2d}\chi^2(d)$. As $d\rightarrow\infty$, both variables approach $\mathcal{N}(\frac{1}{2},\frac{1}{2d})$ and $\mathbf{c}[i]\sim\mathcal{N}(0,\frac{1}{2d}+\frac{1}{2d})=\mathcal{N}(0,\frac{1}{d})$, thus the coordinates of $\mathbf{c}$ are independent and distributed as $\mathcal{N}(0,1/d)$.
\end{proof}

\begin{lemma}\label{lemma:norm-binding}
    For $\mathbf{a},\mathbf{b}\sim\text{Unif}(S^{d-1})$ independent, the norm of their circular convolution $\mathbf{c}=\mathbf{a}\circledast\mathbf{b}$ distributes as 
    $$\|\mathbf{c}\|\sim\mathcal{N}\left(1,\frac{1}{2d}\right)$$
    as $d\rightarrow\infty$.
\end{lemma}

\begin{proof}
    By Lemma~\ref{lemma:convolution}, as $d\rightarrow\infty$ the components of of $\mathbf{c}$ are independent and distributed as $\mathcal{N}(0,1/d)$. To approximate the norm for large $d$, consider the standardized variables 
    $$Z_i=\frac{\mathbf{c}[i]-\mathbb{E}[\mathbf{c}[i]]}{\sqrt{\text{Var}[\mathbf{c}[i]]}}=\frac{\mathbf{c}[i]}{\sqrt{1/d}}\sim\mathcal{N}(0,1)$$
    Then, 
    $$\|\mathbf{c}\|=\sqrt{\sum_{i=0}^{d-1}\dfrac{1}{d} Z_i^2}=\sqrt{\frac{1}{d}}Y$$
    where $Y=\sqrt{\sum_{i=0}^{d-1} Z_i^2}\sim\chi(d)$.
    By the expectation and variance of the Chi distribution and Stirling's approximation of the Gamma function, we have that 
    \begin{align*}
        \mathbb{E}[Y]&=\sqrt{2}\frac{\Gamma(\frac{d+1}{2})}{\Gamma(\frac{d}{2})}\xrightarrow{d\rightarrow\infty}\sqrt{d-\frac{1}{2}}\approx\sqrt{d} \\
        \text{Var}[Y]&=d-\mathbb{E}[Y]^2\xrightarrow{d\rightarrow\infty}d-(d-\frac{1}{2})=\frac{1}{2}
    \end{align*}
    Thus,
    \begin{align*}
        \mathbb{E}[\|\mathbf{c}\|]&=\sqrt{\frac{1}{d}} \mathbb{E}[Y]=\sqrt{\frac{1}{d}}\sqrt{d}=1 \\
        \text{Var}[\|\mathbf{c}\|]&=\frac{1}{d}\text{Var}[Y]=\frac{1}{2d}.
    \end{align*}
    Since a $\chi(d)$ distribution approaches a normal distribution for large $d$, we conclude that $\|\mathbf{c}\|\sim \mathcal{N}\left(1,\frac{1}{2d}\right)$ as $d\rightarrow\infty$. 
\end{proof}

\begin{lemma}\label{lemma:bundle}
    For $\mathbf{a}_1,\ldots,\mathbf{a}_m,\mathbf{b}_1,\ldots,\mathbf{b}_m\sim\text{Unif}(S^{d-1})$, the norm of the bundle of $m$ bound pairs
    $$\mathbf{r}=\sum_{i=1}^m \text{norm}(\mathbf{a}_i\circledast \mathbf{b}_i),$$
    satisfies that, as $d\rightarrow\infty$,
    $$\|\mathbf{r}\|\sim \mathcal{N}\left(\sqrt{m},\frac{m}{2d}\right).$$
\end{lemma}

\begin{proof}
    By Lemma~\ref{lemma:convolution}, as $d\rightarrow\infty$ the components of each non-normalized binding $\mathbf{a}_i\circledast \mathbf{b}_i$ in $\mathbf{r}$ are independent and distributed as $\mathcal{N}(0,1/d)$. As described in \cite{Vershynin_2018}, the normalization of such a vector is uniform on the unit sphere $S^{d-1}$, i.e.
    $$\frac{\mathbf{a}_i\circledast \mathbf{b}_i}{\|\mathbf{a}_i\circledast \mathbf{b}_i\|}\sim\text{Unif}(S^{d-1}).$$
    
    Then, since $\mathbf{r}$ is the sum of $m$ independent vectors drawn from the unit sphere, by Lemma~\ref{lemma:sphere} the coordinates of $\mathbf{r}$ are the sum of $m$ independent $\mathcal{N}(0,1/d)$ variables. Therefore, as $d\rightarrow\infty$, the coordinates of $\mathbf{r}$ are independent and distributed as
    $$\mathbf{r}[i]\sim\mathcal{N}\left(0,\frac{m}{d}\right).$$

    To study the norm for large $d$, we proceed analogously to the proof of Lemma~\ref{lemma:norm-binding}. Let us consider the standardized variables 
    $$Z_i=\frac{\mathbf{r}[i]-\mathbb{E}[\mathbf{r}[i]]}{\sqrt{\text{Var}[\mathbf{r}[i]]}}=\frac{\mathbf{r}[i]}{\sqrt{m/d}}\sim\mathcal{N}(0,1)$$
    Then, 
    $$\|\mathbf{r}\|=\sqrt{\sum_{i=0}^{d-1}\dfrac{m}{d} Z_i^2}=\sqrt{\frac{m}{d}}Y$$
    where $Y=\sqrt{\sum_{i=0}^{d-1} Z_i^2}\sim\chi(d)$. By using the expectation and variance of the Chi distribution, we have that
    \begin{align*}
        \mathbb{E}[\|\mathbf{r}\|]&=\sqrt{\frac{m}{d}} \mathbb{E}[Y]=\sqrt{\frac{m}{d}}\sqrt{d}=\sqrt{m} \\
        \text{Var}[\|\mathbf{r}\|]&=\frac{m}{d}\text{Var}[Y]=\frac{m}{2d}.
    \end{align*}
    Since a $\chi(d)$ distribution approaches a normal distribution for large $d$, we conclude that $\|\mathbf{r}\|\sim \mathcal{N}\left(\sqrt{m},\frac{m}{2d}\right)$ as $d\rightarrow\infty$. 
\end{proof}

\begin{proposition}\label{prop:equality-match}
    For $1\leq n\leq m$, let $\mathbf{a}_1,\ldots,\mathbf{a}_m,\mathbf{b}_1,\ldots,\mathbf{b}_m$ and $\mathbf{v}_1,\ldots,\mathbf{v}_{n}$ be sampled uniformly from the $d$-dimensional unit sphere $S^{d-1}$. For row $r=(a_1:b_1,\ldots, a_m:b_m)$ and selection equality predicate $q=(a_{k_1}:v_1,\ldots, a_{k_n}:v_n)$ for $\{k_1,\ldots,k_n\}\subseteq\{1,\ldots,m\}$, consider their encoding 
    $$\mathbf{r}=\text{norm}\left(\sum_{i=1}^m \text{norm}(\mathbf{a}_i\circledast \mathbf{b}_i)\right),\ \ \mathbf{q}=\text{norm}\left(\sum_{j=1}^n \text{norm}(\mathbf{a}_{k_j}\circledast \mathbf{v}_j)\right).$$
    Then, as $d\rightarrow\infty$, the expectation and variance of their similarity over the random choice of atomic vectors satisfy
    \begin{align*}
        \mathbb{E}[\langle\mathbf{r},\mathbf{q}\rangle|\ \text{equality match}]&=\sqrt{\dfrac{n}{m}}\\
        \text{Var}[\langle\mathbf{r},\mathbf{q}\rangle|\ \text{equality match}]&\leq\frac{8dm-8d+8m+4dn+n-8}{4md^2}.\\
    \end{align*}
\end{proposition}

\begin{proof}
    By the linearity of the inner product, we can write
    \begin{align}
        \langle\mathbf{r},\mathbf{q}\rangle=&\left\langle\text{norm}(\sum_{i=1}^m \text{norm}(\mathbf{a}_i\circledast \mathbf{b}_i)),\ \text{norm}(\sum_{j=1}^n \text{norm}(\mathbf{a}_{k_j}\circledast \mathbf{v}_j))\right\rangle\nonumber\\
        &=\frac{1}{\|\mathbf{r}'\|\cdot \|\mathbf{q}'\|}\sum_{i=1}^m\sum_{j=1}^n \left\langle\text{norm}(\mathbf{a}_i\circledast \mathbf{b}_i),\ \text{norm}(\mathbf{a}_{k_j}\circledast \mathbf{v}_j)\right\rangle\nonumber \\
        &=\frac{1}{\|\mathbf{r}'\|\cdot \|\mathbf{q}'\|}\sum_{i=1}^m\sum_{j=1}^n \frac{\left\langle\mathbf{a}_i\circledast \mathbf{b}_i,\ \mathbf{a}_{k_j}\circledast \mathbf{v}_j\right\rangle}{\|\mathbf{a}_i\circledast \mathbf{b}_i\|\cdot\|\mathbf{a}_{k_j}\circledast \mathbf{v}_j\|} \label{eq:product-expression}
    \end{align}
    where $\mathbf{r}'$ and $\mathbf{q}'$ are the unnormalized bundles of normalized bound pairs for $r$ and $q$ respectively. For simplicity, we will denote them as
    $$B_m=\|\mathbf{r}'\|\qquad B_n=\|\mathbf{q}'\|$$
    Moreover, since expression~\eqref{eq:product-expression} contains $mn$ terms, we can consider an arbitrary ordering of these terms and denote each one as 
    $$\frac{X_i}{D_i},\quad 1\leq i\leq mn$$
    where $X_i$ is the dot product of the $i$-th pair of bound pairs, and $D_i$ is the product of the norms of the bound pairs in that term. With this notation, we can write
    \begin{equation}\label{eq:factorized-expression}
        \langle\mathbf{r},\mathbf{q}\rangle=\frac{1}{B_m\cdot B_n}\sum_{i=1}^{mn} \frac{X_i}{D_i}.
    \end{equation}

    If  row $r$ and predicate $q$ have an equality match, then there are $n$ pairs (from the $mn$ total) that have the form
    $$\frac{X_i}{D_i}=\frac{\langle\mathbf{a}\circledast\mathbf{b},\mathbf{a}\circledast\mathbf{b}\rangle}{\|\mathbf{a}\circledast\mathbf{b}\|^2}=\frac{\|\mathbf{a}\circledast\mathbf{b}\|^2}{\|\mathbf{a}\circledast\mathbf{b}\|^2}=1$$
    thus they are constant.

    To obtain the expectation and variances for the terms of non-matching pairs, we have
    \begin{itemize}
        \item From Lemma~\ref{lemma:inner-product}, 
        $$\mathbb{E}[X_i]=0\qquad \text{Var}[X_i]\leq\frac{2d+2}{d^2},$$
        where the upper bound comes from the case in which the two bound pairs share a common vector.
        \item $D_i=C_{i1}\cdot C_{i2}$, for $C_{i1}, C_{i2}$ defined as the norm of a bound pair. From Lemma~\ref{lemma:norm-binding}, we know $C_{i1}, C_{i2}\sim\mathcal{N}(1,\frac{1}{2d})$, so that they highly concentrate around their mean as $d\rightarrow\infty$. Therefore, we can use the delta method~\cite{rice2007mathematical} to obtain a first order approximation for the expectation of $D_i$ is \begin{align*}
            \mathbb{E}[C_{i1}\cdot C_{i2}]&\approx\mathbb{E}[C_{i1}]\cdot\mathbb{E}[C_{i2}]=1
\end{align*}
    \end{itemize}
    From this, since $X_i$ and $D_i$ are independent are highly concentrated around their mean, we can use the delta method again on $X_i/D_i$.
    \begin{align*}
        \mathbb{E}\left[\frac{X_i}{D_i}\right]&\approx\frac{\mathbb{E}[X_i]}{\mathbb{E}[D_i]}=0\\
        \text{Var}\left[\frac{X_i}{D_i}\right]&\approx\left(\frac{\mathbb{E}[X_i]}{\mathbb{E}[D_i]}\right)^2\left(\frac{\text{Var}[X_i]}{\mathbb{E}[X_i]^2}+\frac{\text{Var}[D_i]}{\mathbb{E}[D_i]^2}\right)\\
        & = \frac{\text{Var}[X_i]}{\mathbb{E}[D_i]^2}+\frac{\mathbb{E}[X_i]^2}{\mathbb{E}[D_i]^4}\text{Var}[D_i]\\
        & \leq \frac{2d+2}{d^2}
    \end{align*}

    Given the $mn$ terms in expression~\eqref{eq:factorized-expression}, we now have that their sum $S$ satisfies
    $$S=\sum_{i=1}^{mn}\frac{X_i}{D_i}=n+\sum_{j=1}^{mn-n}\frac{X_j}{D_j}$$
    where the remaining terms represent non-matching pairs. With this, given that variances are $\mathcal{O}(\frac{1}{d})$, each term concentrates around its mean, and we have expressions for the expectation and variance of $S$ as
    $$\mathbb{E}[S]=n\qquad \text{Var}[S]\leq\frac{n(m-1)(2d+2)}{d^2}.$$

    For the denominator of expression~\eqref{eq:factorized-expression}, by Lemma~\ref{lemma:bundle} we have that $B_m\sim\mathcal{N}(\sqrt{m},\frac{m}{2d})$ and $B_n\sim\mathcal{N}(\sqrt{n},\frac{n}{2d})$. Then, 
    \begin{align*}
        \mathbb{E}[B_mB_n]&=\mathbb{E}[B_m]\cdot\mathbb{E}[B_n]\\
        &=\sqrt{mn}\\
        \text{Var}[B_mB_n]&=\text{Var}[B_m]\mathbb{E}[B_n]^2+\text{Var}[B_n]\mathbb{E}[B_m]^2\\
        &\ \ \ +\text{Var}[B_m]\text{Var}[B_n]\\
        &=\frac{mn(4d+1)}{4d^2}
    \end{align*}
    Finally, we apply the delta method for the desired variable
    $$\langle\mathbf{r},\mathbf{q}\rangle=\frac{S}{B_m\cdot B_n}.$$
    We obtain,
    \begin{align*}
        \mathbb{E}\left[\frac{S}{B_mB_n}\right]&\approx\frac{\mathbb{E}[S]}{\mathbb{E}[B_mB_n]}=\frac{n}{\sqrt{mn}}=\sqrt{\dfrac{n}{m}}\\
        \text{Var}\left[\frac{S}{B_mB_n}\right]&\approx\left(\frac{\mathbb{E}[S]}{\mathbb{E}[B_mB_n]}\right)^2\left(\frac{\text{Var}[S]}{\mathbb{E}[S]^2}+\frac{\text{Var}[B_mB_n]}{\mathbb{E}[B_mB_n]^2}\right)\\
        &\leq\left(\frac{n}{\sqrt{nm}}\right)^2\left(\frac{n(m-1)(2d+2)}{d^2}\frac{1}{n^2}+\frac{mn(4d+1)}{4d^2}\frac{1}{mn}\right)\\
        &=\frac{n}{m}\left(\frac{(m-1)(2d+2)}{nd^2}+\frac{(4d+1)}{4d^2}\right)\\
        & = \frac{8dm-8d+8m+4dn+n-8}{4md^2}
    \end{align*}

\end{proof}

\begin{proposition}
    For $1\leq n\leq m$, let $\mathbf{a}_1,\ldots,\mathbf{a}_m,\mathbf{b}_1,\ldots,\mathbf{b}_m$ and $\mathbf{v}_1,\ldots,\mathbf{v}_{n}$ be sampled uniformly from the $d$-dimensional unit sphere $S^{d-1}$. For row $r=(a_1:b_1,\ldots, a_m:b_m)$ and selection non-equality predicate $q=(a_{k_1}:v_1,\ldots, a_{k_n}:v_n)$ for $\{k_1,\ldots,k_n\}\subseteq\{1,\ldots,m\}$, consider their encoding 
    $$\mathbf{r}=\text{norm}\left(\sum_{i=1}^m \text{norm}(\mathbf{a}_i\circledast \mathbf{b}_i)\right),\ \ \mathbf{q}=\text{norm}\left(\sum_{j=1}^n \text{norm}(\mathbf{a}_{k_j}\circledast -\mathbf{v}_j)\right),$$
    as $d\rightarrow\infty$, the expectation and variance of their similarity over the random choice of atomic vectors satisfy
    \begin{align*}
        \mathbb{E}[\langle\mathbf{r},\mathbf{q}\rangle|\ \text{non-equality match}]&=0\\
        \text{Var}[\langle\mathbf{r},\mathbf{q}\rangle|\ \text{non-equality match}]&\leq \frac{2d+2}{d^2}.\\
    \end{align*}
\end{proposition}

\begin{proof}
    The sign change in the encoding does not modify the distributions used in the proof of Proposition~\ref{prop:equality-match}.

    If $r$ and $q$ have a non-equality match, then for every $j$ there is an $i$ such that $a_i=a_{k_j}$ while $b_i\neq v_i$. Since $\mathbf{b}_i,\mathbf{v}_j$ are independent and sampled from $\text{Unif}(S^{d-1})$ and the uniform unit sphere is isotropic (Proposition 3.3.8. in \cite{Vershynin_2018}), $\mathbf{b}_i$ and $-\mathbf{v}_j$ are also independent with the same distribution. Therefore, the hypotheses of Lemma~\ref{lemma:inner-product} apply. Then, expression~\ref{eq:factorized-expression} consists only of non-matching pairs. Defining
    $$S=\sum_{i=1}^{mn}\frac{X_i}{D_i},$$
    using the same procedure as in the proof of Proposition~\ref{prop:equality-match}, we obtain the following results for the sum of $mn$ non-matching pairs
    $$\mathbb{E}[S]=0\qquad \text{Var}[S]\leq\frac{mn(2d+2)}{d^2}.$$
    Finally, we obtain
    \begin{align*}
        \mathbb{E}\left[\frac{S}{B_mB_n}\right]&\approx\frac{\mathbb{E}[S]}{\mathbb{E}[B_mB_n]}=0\\
        \text{Var}\left[\frac{S}{B_mB_n}\right]&\approx\left(\frac{\mathbb{E}[S]}{\mathbb{E}[B_mB_n]}\right)^2\left(\frac{\text{Var}[S]}{\mathbb{E}[S]^2}+\frac{\text{Var}[B_mB_n]}{\mathbb{E}[B_mB_n]^2}\right)\\
        & = \frac{\text{Var}[S]}{\mathbb{E}[B_mB_n]^2}+\frac{\mathbb{E}[S]^2}{\mathbb{E}[B_mB_n]^4}\text{Var}[B_mB_n]\\
        & \leq \frac{2d+2}{d^2}\frac{1}{mn}+\frac{0}{m^2n^2}\frac{mn(4d+1)}{4d^2}\\
        & = \frac{2d+2}{d^2}
    \end{align*}
\end{proof}

\section{Complete results}

\begin{figure}[h!]
    \centering
    \includegraphics[width=0.47\textwidth]{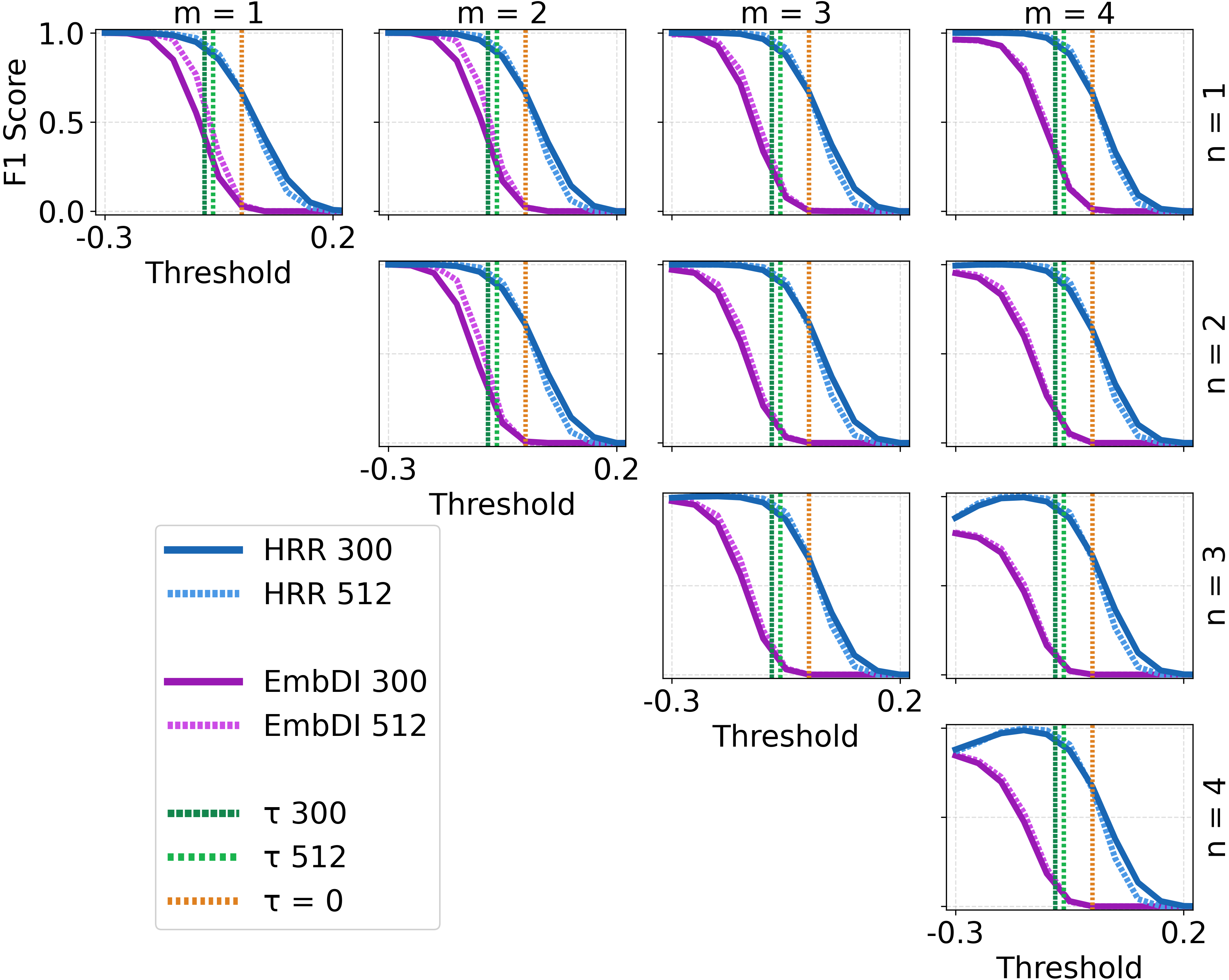}
    \caption{F1-score vs similarity threshold for row retrieval, for all table sizes $m$ and predicate lengths $n$ for the DBLP dataset on non-equality predicates.}
    \label{fig:grid_nonequality_dblp}
\end{figure}

\begin{figure*}
    \centering
    \includegraphics[width=1\textwidth]{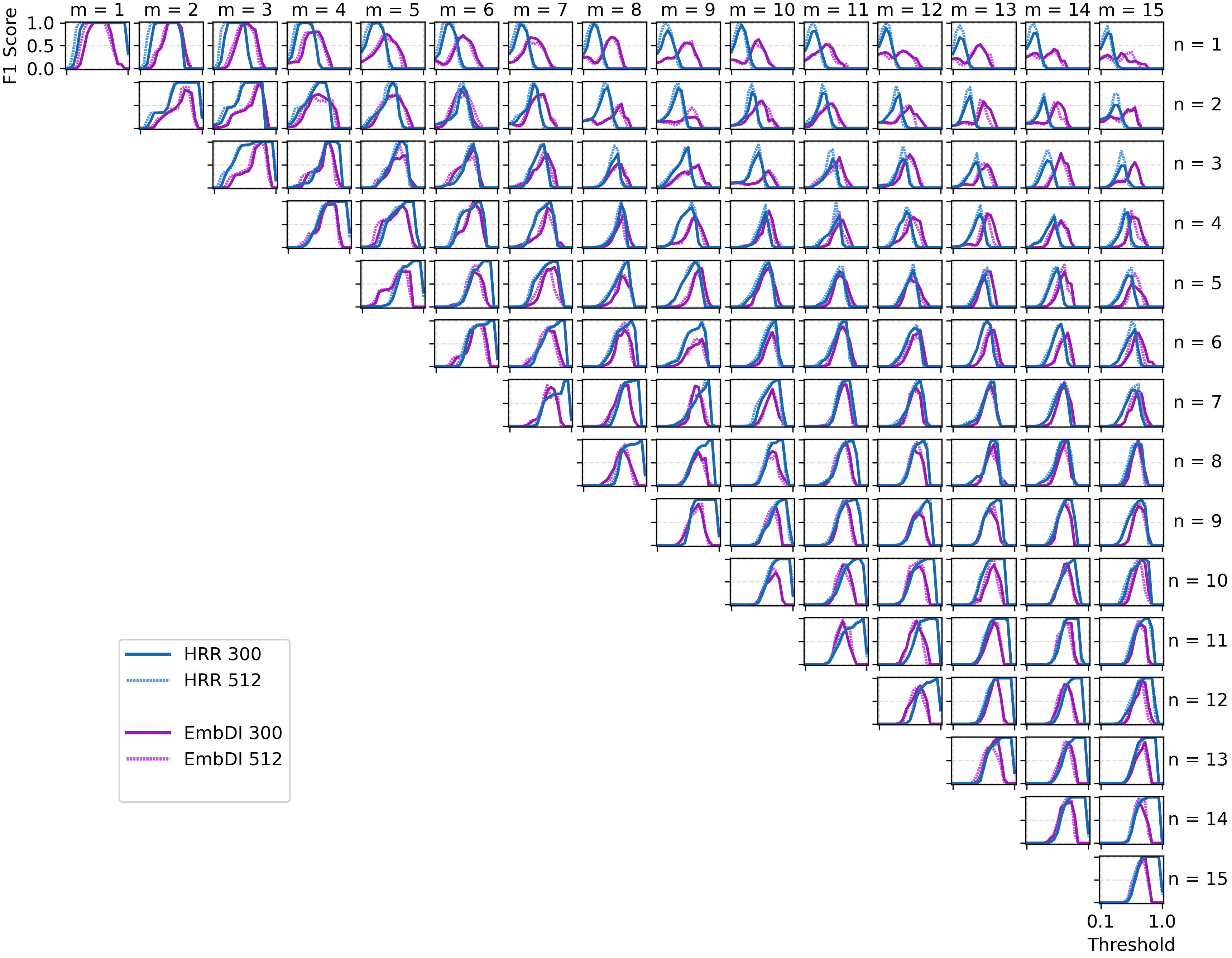}
    \caption{F1-score vs similarity threshold for row retrieval, for all table sizes $m$ and predicate lengths $n$ for the Movie dataset on equality predicates. All subplots share the same axes and scale.}
    \label{fig:grid_equality_movie}
\end{figure*}

\begin{figure*}
    \centering
    \includegraphics[width=1\textwidth]{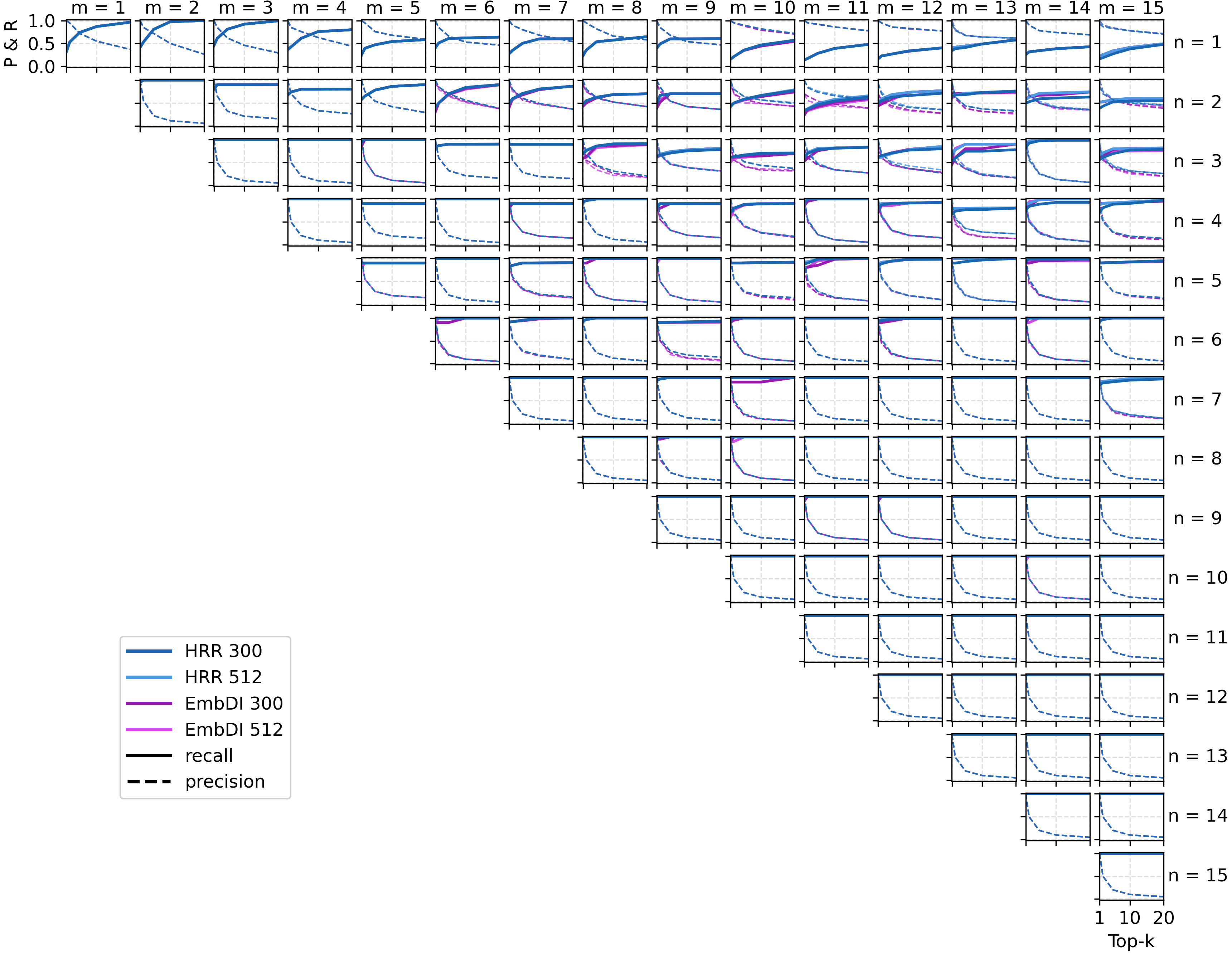}
    \caption{Precision and recall vs top-$k$ for row retrieval, for all table sizes $m$ and predicate lengths $n$ for the Movie dataset on equality predicates. The possibles values are $k\in\{1,2,5,10,20\}$. All subplots share the same axes and scale.}
    \label{fig:topk_equality_movie}
\end{figure*}

\begin{figure*}
    \centering
    \includegraphics[width=1\textwidth]{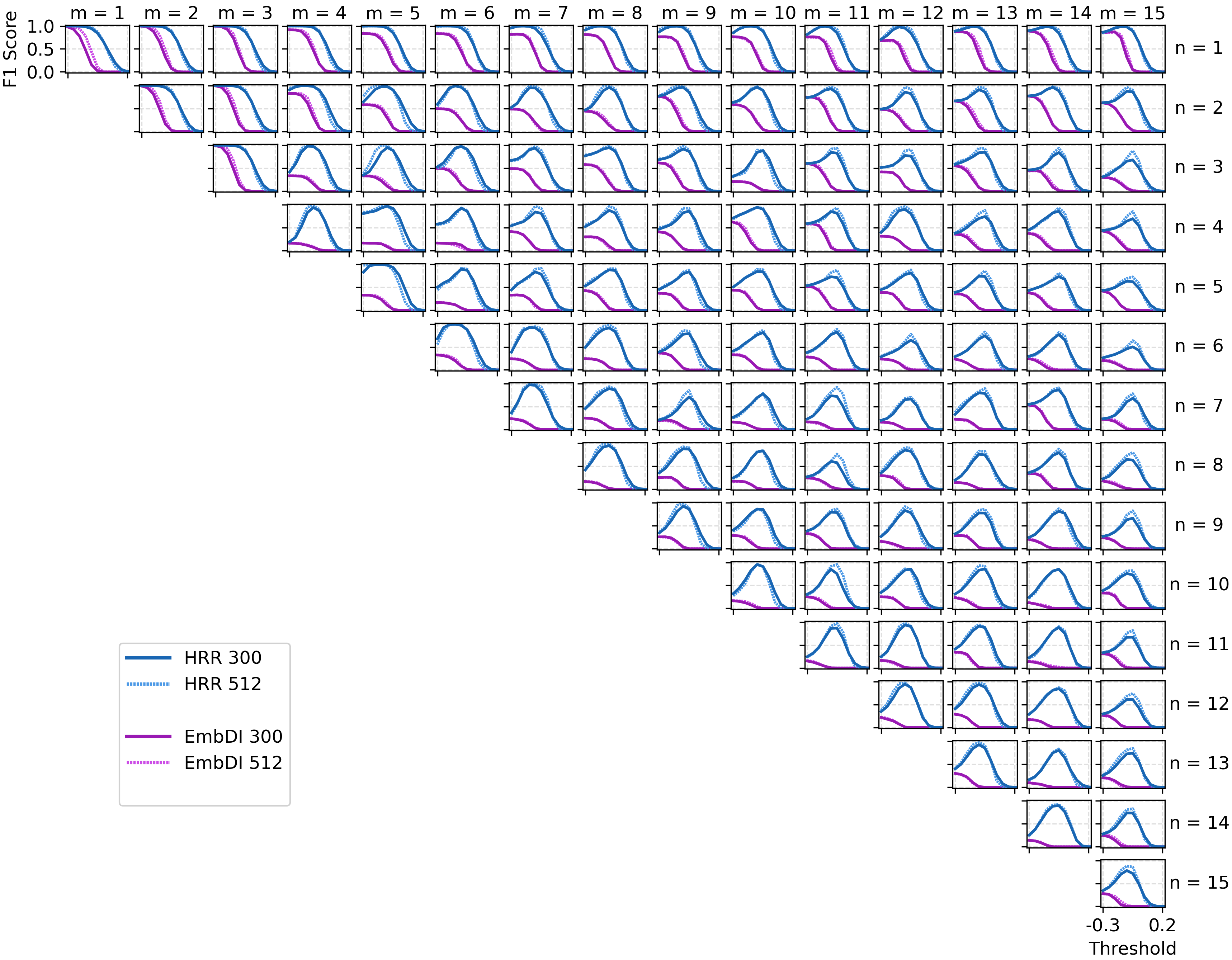}
    \caption{F1-score vs similarity threshold for row retrieval, for all table sizes $m$ and predicate lengths $n$ for the Movie dataset on non-equality predicates. All subplots share the same axes and scale.}
    \label{fig:grid_nonequality_movie}
\end{figure*}

\end{appendices}

\end{document}